\documentclass{article}
\usepackage{times}
\usepackage{natbib}
\usepackage{color}
\usepackage{url}
\usepackage{bbm}
\usepackage[top=1in,bottom=1in,left=1in,right=1in]{geometry}
\usepackage{graphicx}
\usepackage{amssymb,amsmath,amsthm}

\usepackage{algorithm}
\usepackage{algpseudocode}

\usepackage{booktabs}
\usepackage{enumitem}

\graphicspath{{FIGS/}}

\newtheorem{lemma}{Lemma}
\newtheorem{theorem}{Theorem}
\newtheorem{corollary}{Corollary}
\newtheorem{prop}{Proposition}
\newtheorem{asm}{Assumption}

\newcommand*{\QEDA}{\hfill\ensuremath{\square}}

\title{Convex Relaxation Regression: Black-Box Optimization of Smooth Functions by Learning Their Convex Envelopes}

\author{ {\bf Mohammad Gheshlaghi Azar, Eva L. Dyer, and Konrad P. K\"ording } \\
Dept. of Physical Medicine and Rehabilitation, Northwestern University
}

\begin{document}

\maketitle

\begin{abstract}
Finding efficient and provable methods to solve non-convex optimization problems is an outstanding challenge in machine learning and optimization theory. A popular approach used to tackle non-convex problems is to use convex relaxation techniques to find a convex surrogate for the problem. Unfortunately, convex relaxations typically must be found on a problem-by-problem basis. Thus, providing a  general-purpose strategy to  estimate a convex relaxation would have a wide reaching impact. Here, we introduce  \emph{Convex  Relaxation Regression} (CoRR), an approach for learning convex relaxations for a  class of smooth functions. The main idea behind our approach is to estimate the convex envelope of a function $f$ by evaluating $f$ at a set of $T$ random points and then fitting a convex function to these function evaluations. We prove that with probability greater than $1-\delta$, the solution of our algorithm converges to the  global optimizer of $f$ with error $\mathcal{O} \Big( \big(\frac{\log(1/\delta) }{T} \big)^{\alpha} \Big)$ for some $\alpha> 0$. Our approach enables the use of convex optimization tools to solve a  class of non-convex optimization problems.
\end{abstract}

\newcommand{\A}{\mathcal A}
\renewcommand{\S}{\mathcal S}
\newcommand{\TT}{\mathsf T}
\newcommand{\X}{\mathcal X}
\newcommand{\D}{\mathcal D}
\newcommand{\G}{\mathcal G}
\newcommand{\K}{\mathcal K}
\newcommand{\calP}{\mathcal P}
\newcommand{\calI}{\mathcal I}
\newcommand{\barH}{\overline{H}}
\newcommand{\hh}{\hat h}
\renewcommand{\L}{\mathcal L}
\newcommand{\U}{\mathcal U}
\renewcommand{\O}{\mathcal O}
\newcommand{\Hyp}{\mathcal H}
\newcommand{\Y}{\mathcal Y}
\newcommand{\B}{\mathcal B}
\newcommand{\C}{\mathcal C}
\newcommand{\F}{\mathcal F}
\newcommand{\Z}{\mathcal Z}
\newcommand{\E}{\mathbb E}
\newcommand{\calE}{\mathcal E}
\newcommand{\calS}{\mathcal{S}}
\newcommand{\Cov}{\textnormal{Cov}}
\newcommand{\V}{\mathbb V}
\newcommand{\Prob}{\mathbb P}
\newcommand{\I}{\mathbb I}
\newcommand{\on}{\mathbbm 1}
\newcommand{\aon}{\mathbf 1}
\newcommand{\N}{\mathcal N}
\newcommand{\tM}{\widetilde{M}}
\newcommand{\T}{\mathcal{T}}
\newcommand{\balpha}{\boldsymbol \alpha}
\newcommand{\bmu}{\boldsymbol \mu}
\newcommand{\bSigma}{\boldsymbol \Sigma}
\newcommand{\bP}{\mathbf{P}}
\newcommand{\bhP}{\widehat{\mathbf{P}}}
\newcommand{\bT}{\mathbf{T}}
\newcommand{\bX}{\mathbf{X}}
\newcommand{\bY}{\mathbf{Y}}
\newcommand{\bx}{\boldsymbol{x}}
\newcommand{\bw}{\boldsymbol{w}}
\newcommand{\by}{\boldsymbol{y}}
\newcommand{\bU}{\mathbf{U}}
\newcommand{\bV}{\mathbf{V}}
\newcommand{\MV}{\textnormal{MV}}
\newcommand{\hMV}{\widehat{\textnormal{MV}}}
\newcommand{\barmu}{\bar\mu}
\newcommand{\hpi}{\hat\pi}
\newcommand{\tDelta}{\widetilde{\Delta}}
\newcommand{\hmu}{\widehat{\mu}}
\newcommand{\hrho}{\hat\rho}
\newcommand{\hnu}{\hat\nu}
\newcommand{\trho}{\tilde\rho}
\newcommand{\brho}{\bar\rho}
\newcommand{\hM}{\widehat{M}}
\newcommand{\tmu}{\widetilde{\mu}}
\newcommand{\tpi}{\widetilde{\pi}}
\newcommand{\barvar}{\bar\sigma^2}
\newcommand{\tvar}{\tilde\sigma^2}
\newcommand{\R}{\mathbb{R}}
\newcommand{\htheta}{\hat{\theta}}
\newcommand{\hR}{\widehat{\mathcal{R}}}
\newcommand{\tR}{\widetilde{\mathcal R}}
\newcommand{\invdelta}{1/\delta}
\newcommand{\boldR}{\mathbb R}
\newcommand{\eps}{\varepsilon}
\newcommand{\mvlcb}{\textnormal{\texttt{MV-LCB }}}
\newcommand{\ucb}{\textnormal{\textsl{UCB }}}
\newcommand{\ucbv}{\textnormal{\textsl{UCB-V }}}
\newcommand{\mvlcbt}{\textnormal{\texttt{MV-LCB(t) }}}
\newcommand{\mom}{\textnormal{MoM}}
\newcommand{\me}{\textnormal{ME}}
\newcommand{\mt}{\textnormal{MT}}
\newcommand{\eh}{1/(1-\gamma)}
\newcommand{\ehf}{\frac1{1-\gamma}}

\newcommand{\cvar}{\textnormal{C}}
\newcommand{\hcvar}{\widehat{\textnormal{C}}}
\newcommand{\var}{\textnormal{V}}
\newcommand{\hvar}{\widehat{\textnormal{V}}}

\newcommand{\avg}[2]{\frac{1}{#2} \sum_{#1=1}^{#2}}
\newcommand{\hDelta}{\widehat{\Delta}}
\newcommand{\hGamma}{\widehat{\Gamma}}

\newcommand{\todo}[1]{{\bf[#1]}}

\newcommand{\beq}{\begin{equation}}
\newcommand{\eeq}{\end{equation}}

\newcommand{\beqa}{\begin{eqnarray}}
\newcommand{\eeqa}{\end{eqnarray}}

\newcommand{\beqan}{\begin{eqnarray*}}
\newcommand{\eeqan}{\end{eqnarray*}}

\renewcommand{\P}{\mathbb{P}}
\renewcommand{\Pr}{\mathbb{P}}
\newcommand{\Q}{\mathbb{Q}}
\newcommand{\Esp}{\mathbb{E}}
\newcommand{\Var}{\mathbb{V}}
\newcommand{\indic}[1]{\mathbb{I}\{#1\}}
\newcommand{\EE}[1]{\E\left[#1\right]}
\newcommand{\wh}{\widehat}
\newcommand{\wt}{\widetilde}
\newcommand{\BlackBox}{\rule{1.5ex}{1.5ex}} 
\newcommand{\supp}{\mathop{\mathrm{supp}}}

\let\R\undefined 
\newcommand{\R}{\mathbb{R}}
\newcommand{\Real}{\mathbb{R}}
\newcommand{\Normal}{\mathcal{N}}

\newcommand{\eqdef}{\stackrel{\rm def}{=}}

\newcommand{\cl}[2][ (]{
\ifthenelse{\equal{#1}{ (}}{\left (#2\right)}{}
\ifthenelse{\equal{#1}{[}}{\left[#2\right]}{}
\ifthenelse{\equal{#1}{\{}}{\left\{#2\right\}}{}
}

\newcommand{\inset}[3][C]{
\ifthenelse{\equal{#1}{C}}{{#2}\in\mathcal{#3}}{}
\ifthenelse{\equal{#1}{T}}{{#2}\in{#3}}{}
}

\newcommand{\ESum}[4][C]
{\ifthenelse{\equal{#1}{C}}{\underset{\inset{#3}{#4}}{\sum}#2}{}
 \ifthenelse{\equal{#1}{T}}{\underset{\inset[N]{#3}{#4}}{\sum}#2}{}
\ifthenelse{\equal{#1}{U}}{\overset{#4}{\underset{#3}{\sum}}#2}{}
\ifthenelse{\equal{#1}{X}}{\sideset{}{_{#3}^{#4}}{\sum}#2}{}
\ifthenelse{\equal{#1}{S}}{\sideset{}{_{#3}^{#4}}{\sum}#2}{}
\ifthenelse{\equal{#1}{O}}{{\underset{ (#3,#4)}{\sum}}#2}{}
\ifthenelse{\equal{#1}{I}}{{\underset{#3}{\sum}}#2}{}}

\newcommand{\VF}[2][N]{
\ifthenelse{\equal{#1}{L}}{V^{\pi}_{\lambda} (#2)}{}
\ifthenelse{\equal{#1}{C}}{V{^{\pi} (#2)}}{}
\ifthenelse{\equal{#1}{T}}{V^*_{\bar{\pi}} (#2)}{}
\ifthenelse{\equal{#1}{CO}}{V{^{*} (#2)}}{}
\ifthenelse{\equal{#1}{Mx}}{V^{\pi}_{\infty} (#2)}{}
\ifthenelse{\equal{#1}{Mxo}}{V^{\pi^*}_{\infty} (#2)}{}
\ifthenelse{\equal{#1}{Mn}}{V^{\pi}_{-\infty} (#2)}{}
\ifthenelse{\equal{#1}{Mno}}{V^{\pi^*}_{-\infty} (#2)}{}
}

\newcommand{\Eval}[1][null]{
\ifthenelse{\equal{#1}{null}}{\mathbb{E}}{\mathbb{E}_{#1}}
}

\newcommand{\M}[1][]{
\mathcal{M}_{#1}
}

\newcommand{\qv}[1][null]{
\ifthenelse{\equal{#1}{null}}{Q^*}{Q^{#1}}
}

\newcommand{\subLim}[2]{
\underset{#1\rightarrow#2}{\lim}
}

\newcommand{\Norm}[2][]
{
\left\|#2\right\|_{#1}
}

\newcommand{\bldsym}{\boldsymbol}

\newtheorem{assumption}{Assumption}

\newcommand{\TODO}[1]{(\textbf{TODO: {#1}})}

\def\undertilde#1{\mathord{\vtop{\ialign{##\crcr
$\hfil\displaystyle{#1}\hfil$\crcr\noalign{\kern1.5pt\nointerlineskip}
$\hfil\tilde{}\hfil$\crcr\noalign{\kern1.5pt}}}}}

\section{Introduction}
Modern machine learning relies heavily on optimization techniques to extract information from large and noisy datasets \citep{friedman2001elements}. Convex optimization methods are widely used in machine learning applications, due to fact that convex problems can be solved efficiently, often with a first order method such as  gradient descent \citep{shalev2014understanding,sra2012optimization,boyd2004convex}. A wide class of problems can be cast as convex optimization problems; however, many important learning problems, including  binary classification with 0-1 loss, sparse and low-rank matrix recovery, and training multi-layer neural networks, are non-convex. 

In many cases, non-convex optimization problems can be solved by first relaxing the problem: {\em convex relaxation} techniques find a convex function that approximates the original objective function \citep{tropp2006algorithms,candes2010power,chandrasekaran2012convex}.  A convex relaxation is considered tight when it provides a tight lower bound to the original objective function. Examples of problems for which tight convex relaxations are  known include binary classification  \citep{cox1958regression},  sparse and low-rank approximation \citep{tibshirani1996regression, recht2010guaranteed}. The recent success of both sparse and low rank matrix recovery has demonstrated the power of convex relaxation for solving high-dimensional machine learning problems.



When a tight convex relaxation is known, then the underlying non-convex problem can often be solved by optimizing its convex surrogate in lieu of the original non-convex problem. However, there are important classes of machine learning problems for which no such relaxation is known. 
These include a wide range of machine learning problems such as training deep neural nets, estimating latent variable models (mixture density models), optimal control, reinforcement learning, and hyper-parameter optimization. Thus, methods for finding convex relaxations of arbitrary non-convex functions would have wide reaching impacts throughout machine learning and the computational sciences.
 
Here we introduce a  principled approach for black-box (zero-order) global optimization that is based on learning a convex relaxation to a non-convex function of interest (Sec.\ \ref{sec:alg}). To motivate our approach, consider the problem of estimating the convex envelope of the function $f$,  i.e., the tightest convex lower bound of the function \citep{grotzinger1985supports,falk1969lagrange,kleibohm1967bemerkungen}. In this case, we know that the envelope's minimum coincides with the minimum of the original non-convex function \citep{kleibohm1967bemerkungen}. Unfortunately, finding the \emph{exact} convex envelope of a non-convex function can be at least as hard as solving the original optimization problem. This is due to the fact that the problem of finding the convex envelope of a function is equivalent to the problem of computing its Legendre-Fenchel bi-conjugate  \citep{rockafellar1997convex,falk1969lagrange}, which is in general as hard as optimizing $f$. Despite this result, we show that for a class of smooth (non-convex) functions, it is possible to accurately and efficiently \emph{estimate} the convex envelope from a set of function evaluations. 

The main idea behind our approach, {\em Convex Relaxation Regression} (CoRR), is to empirically estimate the convex envelope of $f$ and then optimize the resulting empirical convex envelope. We do this by solving a constrained  $\ell_1$ regression problem which estimates the convex envelope by a linear combination of a set of convex functions (basis vectors). As our approach only requires samples from the function, it can be used to solve black-box optimization problems where gradient information is unknown. Whereas most methods for global optimization rely on local search strategies which find a new search direction to explore, CoRR takes a global perspective: it aims to form a global estimate of the function to ``fill in the gaps'' between samples. Thus CoRR provides an efficient strategy for global minimization through the use of convex optimization tools. 

One of the main theoretical contributions of this work is the development of guarantees that CoRR can find accurate convex relaxations for a broad class of non-convex functions (Sec.\ \ref{sec:theory}). We prove in Thm.\ \ref{thm:f.bound.exact}  that with probability greater than $1-\delta$, we can approximate the global minimizer with error of $\mathcal{O}\Big( \big(\frac{\log(1/\delta) }{T} \big)^{\alpha} \Big)$, where $T$ is the number of function evaluations and $\alpha>0$ depends upon the exponent of the H\"older-continuity bound on $f(x) - f^*$. This result assumes that the true convex envelope lies in the function class used to form a convex approximation. In Thm.\ \ref{thm:f.bound.approx}, we extend this result for the case where the convex envelope is  in the proximity of this set of functions. Our results may also translated to a bound with polynomial dependence on the dimension (Sec.\ \ref{sec:dimension}). 


The main contributions of this work are as follows. We introduce CoRR, a method for black-box optimization that learns a convex relaxation of a function from a set of random function evaluations (Sec.\ \ref{sec:alg}). Following this, we provide performance guarantees which show that as the number of function evaluations $T$ grows, the error decreases polynomially in $T$ (Sec.\ \ref{sec:theory}). In Thm.\ \ref{thm:f.bound.exact} we provide a general result for the case where the true convex envelope $f_c$ lies in the function class $\Hyp$ and extend this result to the approximate setting where $f_c \notin \Hyp$ in Thm.\ \ref{thm:f.bound.approx}. Finally, we study the performance of CoRR on several multi-modal test functions and compare it with a number of widely used approaches for global optimization (Sec.\ \ref{sec:numeric}). These results suggest that CoRR can accurately find a tight convex lower bound for a wide class of non-convex functions.


\section{Problem Setup}

We now introduce relevant notation, setup our problem, and then provide background on global optimization of non-convex functions.
\label{sec:setup}
\subsection{Preliminaries}ß
Let $n$  be a positive integer. For every  $x \in\R^n$, its $\ell_2$-norm is denoted by $\|x\|$, where $\|x\|^2:=\langle x,x \rangle$ and $\langle x,y \rangle$ denotes the inner product between two vectors $x\in \R^n$ and $y\in\R^n$. We denote the $\ell_2$ metric by $d_2$ and the set of $\ell_2$-normed  bounded vectors in $\R^n$ by $\B(\R^n)$, where for every $x\in \B(\R^n)$ we assume  that there exists some finite scalar $C$ such that $\|x\|< C$.  Let  $(\X, d) $ be a metric space, where $\X\in\B(\R^n)$ is a  convex set of bounded vectors and $d(.,x)$ is convex w.r.t. its first argument for every $x\in \B(\R^n)$.\footnote{This also implies that $d(x,.)$ is convex w.r.t. its second argument argument for every $x\in \B(\R^n)$ due to the fact that the metric $d$ by definition is symmetric.} We denote the set of all  bounded functions on $\X$ by $\B(\X,\R)$, such that for every $f \in  \B(\X, \R)$ and $x\in \X$ there exists some finite scalar $C>0$ such that  $|f(x)|\leq C$. Finally, we denote the  set of all convex bounded functions on $\X$ by $\C(\X,\R)\subset\B(\X,\R)$. Also for every $\Y\subseteq\B(\R^n)$, we denote the convex hull of $\Y$ by ${\rm conv}(\Y)$. Let $\B(x_0, r)$ denote an open ball of radius $r$ centered at $x_0$.  Let $\mathbbm{1}$ denote a vector of ones.

The convex envelope of function $f:\X\to\R$ is denoted by $f_c:\X \to\R$. Let  $\wt \Hyp$  be  the set of all convex functions defined over $\X$ such that  $h(x)\leq f(x)$ for all $x\in\X$. The function $f_c$ is the convex envelope of $f$ if  for every $x\in\X$ \textbf{(a)} $f_c(x)\leq f(x)$, \textbf{(b)}  for every $h\in\wt \Hyp$ the inequality $h(x)\leq f_c(x)$ holds. Convex envelopes are also related to the concepts of the convex hull and the epigraph of a function. For every function  $f:\X\to\R$ the epigraph is defined as
$\mathrm{epi}f=\{(\xi,x):\xi\geq f(x),x\in \X\}.$
 One can then show that the convex envelope of $f$ is obtained by
$f_c(x)=\inf\{\xi:(\xi,x) \in {\rm conv}({\rm epi}f)\},~\forall x\in \X.$

In the sequel, we will generate a set of function evaluations from $f$ by evaluating the function over i.i.d. samples from $\rho$. $\rho$ denotes a probability distribution on $\X$ such that $\rho(x) > 0$ for all $x\in \X$. In addition, we approximate the convex envelope using a function class $\Hyp$ that contains a set of convex functions $h(\cdot ;\theta) \in\Hyp$ parametrized by $\theta \in  \Theta \subseteq \B(\R^p)$. We also assume that every $h\in \Hyp$ can be expressed as a linear combination of  a set of basis  $\phi:\X\to\B(\R^p)$, that is, $h(x;\theta) = \langle\theta,\phi(x)\rangle$ for every $h(\cdot;\theta)\in\Hyp$ and $x\in \X$. 


\subsection{Black-box Global Optimization Setting}
\label{sec:survey.black}
We consider a  \emph{black-box} (zero-order) global optimization setting, where we assume that we do not have access to information about the gradient of the function that we want to optimize. More formally, let $\F\subseteq \B(\X,\R)$ be a class of bounded functions, where  the image of  every $f\in\F$ is bounded by $R$ and $\X$ is a convex set. We consider the problem of finding the global minimum of the function $f$,
\begin{equation}
\label{eq:fmin}
f^*:=\underset {x\in \X}{\min} ~f(x).
\end{equation}
We denote the set of minimizers of $f$ by $\X^*_f\subseteq \X$.   

In the black-box setting,  the optimizer  has only access to the inputs and outputs of the function $f$. 
In this case, we assume that our optimization algorithm is provided with a set of input points $\wh \X=\{x_1,x_2,\dots,x_T\}$ in $\X$ and a sequence of outputs $[f]_{\wh\X}  = \{ f(x_1),f(x_2),\dots,f(x_T) \}$.  Based upon this information, the goal is to find an estimate $\wh x\in\X$, such that the error $ f(\wh x)-f^*$ becomes as small as possible.

\subsection{Methods for Black-box Optimization}
Standard tools that are used in convex optimization, cannot be readily applied to solve non-convex problems as they only converge to local minimizers of the function. Thus, effective global optimization approaches must have a mechanism to avoid getting trapped in local minima. In low-dimensional settings, performing an exhaustive grid search or drawing random samples from the function can be sufficient \citep{bergstra2012random}. However, as the dimension grows, smarter methods for searching for the global minimizer are required. 


{\bf Non-adaptive search strategies.}~~A wide range of global optimization methods are build upon the idea of iteratively creating a deterministic set (pattern) of points at each iteration, evaluating the function over all points in the set, and selecting the point with the minimum value as the next seed for the following iteration \citep{hooke1961direct,lewis1999pattern}. Deterministic pattern search strategies can be extended by introducing some randomness into the pattern generation step. For instance, simulated annealing \citep{kirkpatrick1983optimization} (SA) and genetic algorithms \citep{back1996evolutionary} both use randomized search directions to determine the next place that they will search. The idea behind introducing some noise into the pattern, is that the method can jump out of local minima that deterministic pattern search methods can get stuck in. While many of these search methods work well  in low dimensions, as the dimension of problem grows, these algorithms often become extremely slow due to the curse of dimensionality.

{\bf Adaptive and model-based search.} ~~In higher dimensions, adaptive and model-based search strategies can be used to further steer the optimizer in good search directions \citep{mockus1978application, hutter2009automated}. For instance, recent results in Sequential Model-Based Global Optimization (SMBO) have shown that Gaussian processes are useful priors for global optimization \citep{mockus1978application, bergstra2011algorithms}. In these settings, each search direction is driven by a model (Gaussian process) and updated based upon the local structure of the function. {These techniques, while useful in low-dimension problems, become inefficient  in high-dimensional settings.

Hierarchical search methods take a different approach in exploiting the structure of the data to find the global minimizer \citep{MAL-038, azar2014stochastic,munos2011optimistic}. The idea behind hierarchical search methods is to identify regions of the space with small function evaluations to sample further (exploitation), as well as generate new samples in unexplored regions (exploration). One can show that it is possible to find the global optimum with a finite number of function evaluations using hierarchical search; however, the number of samples needed to achieve a small error increases exponentially with the dimension. For this reason, hierarchical search methods are often not efficient for high-dimensional problems. 


{\bf Graduated optimization.}~~Graduated optimization methods \citep{blake1987visual,yuille1989energy}, are another class of methods for non-convex optimization which have received much attention in recent years \citep{chapelle2010gradient,Dvijotham14,hazan2015graduated,mobahi2015theoretical}. These methods work by locally smoothing the problem, descending along this smoothed objective, and then gradually sharpening the resolution to hone in on the true global minimizer. Recently \citet{hazan2015graduated} introduced a graduated optimization approach that can be applied in the  black-box optimization setting. In this case, they prove that for a class of functions referred to as $\sigma$-nice functions, their approach is guaranteed to converge to an $\eps$-accurate estimate of the global minimizer at a rate of $\mathcal{O}(n^2/\eps^4)$. To the best of our knowledge, this result represents the state-of-the-art in terms of theoretical results for global black-box optimization.



\section{Algorithm}

\label{sec:alg}

In this section, we introduce \emph{Convex  Relaxation Regression} (CoRR), a black-box optimization approach for global minimization of a bounded function $f$. 

\subsection{Overview} 
\label{sec:alg.motiv}
The main idea behind our approach is to estimate the convex envelope $f_c$ of a function and minimize this surrogate in place of our original function. The following result guarantees that the minimizer of $f$ coincides with the minimizer of $f_c$. 
\begin{prop}[\citealt{kleibohm1967bemerkungen}]
\label{prop:convexenv}
Let $f_c$ be the convex envelope of $f:\X\to\R$. Then  (a) $\min_{x\in\X}f_c(x)= f^\ast$ and (b) $\X^\ast_f\subseteq \X^\ast_{f_c}$.
\end{prop}

This result suggests that one can find the minimizer of $f$ by optimizing its convex envelope. Unfortunately, finding the exact convex envelope of a function is difficult in general. However, we will show that, for a certain class of functions, it is possible to estimate the convex envelope accurately from a set of function evaluations. Our aim is to estimate the convex envelope by fitting a convex function to these function evaluations.
 
 The idea of fitting a convex approximation to samples from $f$ is quite simple and intuitive. However, the best unconstrained convex fit to $f$ does not necessarily coincide with  $f_c$.  Determining whether there exists a set of convex constraints under which the best convex fit to $f$ coincides with $f_c$ is an open problem. The following lemma, which is key to efficient optimization of $f$ with CoRR, provides a solution. This lemma transforms our original non-convex optimization problem to a  least-absolute-error regression problem with a convex constraint, which can be solved using convex optimization tools. 





\begin{lemma}
\label{lem:const.opt}
Let every $h\in\Hyp$ and $f$ be $\lambda$-Lipschitz for some $\lambda>0$. Let $ L(\theta) = \E [ |h(x;\theta)-f(x)| ]$ be the expected loss, where the expectation is taken with respect to the distribution $\rho$.  Assume that there exists $\Theta_c\subseteq \Theta$ such that for every $\theta\in \Theta_c$,  $h(x;\theta)=f_c(x)$ for all $x\in \X$.    Consider the following optimization problem:

\begin{equation}
\label{eq:opt.const1}
\theta_\mu=\underset{\theta\in\Theta}{\arg\min}~~ L(\theta) ~~\textrm{s.t. } ~~ \E[ h(x;\theta) ]= \mu.
\end{equation}

Then there exists a scalar $\mu \in [-R, R]$ for which  $\theta_c\in\Theta_c$. In particular, $\theta_c\in\Theta_c$ when $\mu =\E(f_c(x))$.

\end{lemma}
The formal proof of this lemma is provided in \ref{proof:lemma1} (see Supp.\ Materials). We prove this lemma by showing that for every  $\theta\in\Theta$ where $\E[ h(x;\theta) ]=\E[ f_c(x)]$, and for every $\theta_c\in\Theta_c$, the loss $L(\theta)\geq L(\theta_c)$.  Equality is attained only when $\theta \in \Theta_c$. Thus, $f_c$ is the only minimizer of $L(\theta)$ that satisfies the constraint $\E [ h(x;\theta) ] =\E [ f_c(x)]$.  

{\bf Optimizing $\mu$.}~~Lem.~\ref{lem:const.opt} implies that, for a certain choice of $\mu$,  Eqn.~\ref{eq:opt.const1} provides us with the convex envelope $f_c$. However, finding the exact value of  $\mu$ for which this result holds is difficult, as it requires knowledge of the envelope not available to the learner.  Here we use an alternative approach to find $\mu$ which guarantees that the optimizer of $h(\cdot;\theta_{\mu})$ lies in the set of true optimizers $\X^*_f$. Let $x_{\mu}$ denote the minimizer of $h(\cdot;\theta_\mu)$. We find a $\mu$  which minimizes  $f(x_\mu)$:
\begin{equation} 
\label{eq:fmin}
\mu^{\ast} =  \underset{\mu\in [-R,R] }{\arg\min} ~  f( x_{\mu}).
\end{equation} 
Interestingly, one can show that $x_{\mu^\ast}$ lies in the set $\X^*_f$. To prove this, we use the fact that the minimizers of the convex envelope $f_c$ and $f$ coincide (see Prop.\ \ref{prop:convexenv} and Lem.~\ref{lem:opt.convex} in Supp. Material). This implies that $f(x_{\mu_c})=f^*$, where $\mu_c:=\E(f_c(x))$.   It then follows  that  $f^*=f(x_{\mu_c})\geq\min_{\mu\in[-R,R]} f(x_{\mu})=f(x_{\mu^*})$. This combined with the fact that $f^*$ is the minimizer of $f$ implies that $f(x_{\mu^*})=f^*$ and thus $x_{\mu^*}\in\X^*$. 

\subsection{Optimization Protocol}
We now describe how we use the ideas presented in Sec.~\ref{sec:alg.motiv} to implement CoRR (see Alg.\ \ref{algo:CORR} for pseudocode). Our approach for black-box optimization requires two main ingredients: (1) samples from the function $f$ and (2) a function class $\Hyp$ from which we can form a convex approximation $h$. In practice, CoRR is initialized by first drawing two sets of $T$ samples $ \wh \X_1$ and $ \wh \X_2$ from the domain $\X \subseteq \B(\R^n)$ and evaluating $f$ over both of these sets. With these sets of function evaluations (samples) and a function class $\Hyp$ in hand, our aim is to learn an approximation $h(x; \theta)$ to the convex envelope of $f$. Thus for a fixed value of $\mu$, we solve the following constrained optimization problem (see the \texttt{OPT} procedure in Alg.~\ref{algo:CORR}):
\begin{equation}
\label{eq:CORR.core}
\begin{aligned}
\wh \theta_c=\arg\min_{\theta \in \Theta}~ \wh \E_1 \big[ |h(x;\theta)-f(x)| \big]~ \text{s.t.}~ \wh \E_2\big[ h(x;\theta) \big] = \mu,
\end{aligned}
\end{equation}  
where the empirical expectation $\wh \E_i [ g(x) ]:=1/T \sum_{x\in\wh \X_i} g(x),$ for every  $g\in\B(\X,\R)$ and $i\in\{1,2\}$. We provide pseudocode for optimizing Eqn.\ \ref{eq:CORR.core} in the $\texttt{OPT}$ procedure of Alg.\ \ref{algo:CORR}.

The optimization problem of Eqn.~\ref{eq:CORR.core} is an empirical approximation of the optimization problem in Eqn.~\ref{eq:opt.const1}. However, unlike Eqn.~\ref{eq:opt.const1}, in which $L(\theta)$  is not easy to evaluate and optimize,  the empirical loss can be optimized efficiently using standard convex optimization techniques. In addition, one can establish bounds on the error $|L( \wh \theta_c)- L(\theta_c)|$ in terms of the sample size $T$ using standard results from the literature on stochastic convex optimization \citep[see, e.g., Thm.~1 in][]{shalev2009stochastic}. Optimizing the empirical loss provides us with an accurate estimate of the convex envelope as the number of function evaluations increases. 

The search for the best $\mu$ (Step 2 in Alg.\ \ref{algo:CORR}) can be done by solving Eqn.\ \ref{eq:fmin}. As $\mu$ is a scalar with known upper and lower bounds, we can employ a number of hyper-parameter search algorithms \citep{munos2011optimistic,  bergstra2011algorithms} to solve this 1D optimization problem. These algorithms guarantee fast convergence to the global minimizer in low dimensions and thus can be used to efficiently search for the solution to Eqn.\ \ref{eq:fmin}. Let $\wh \mu$ denote the final estimate of $\mu$ obtained in Step 2 of Alg.\ \ref{algo:CORR} and let $h(\cdot; \theta_{\wh \mu})$ denote our final convex approximation to $f_c$. The final solution $\wh x_{\wh \mu}$ is then obtained by optimizing  $h(\cdot;\theta_{\wh \mu})$ (Step 2 of \texttt{OPT}). 

\renewcommand{\algorithmicrequire}{\textbf{Input:}}
\renewcommand{\algorithmicensure}{\textbf{Output:}}

\begin{algorithm}[t!]
  
 \begin{algorithmic}[1]
\Require A black-box function $f$ which returns a sample $f(x)$ when evaluated at a point $x$. The number of samples $N$ to draw from $f$. A class $\Hyp\subseteq\B(\X,\R)$ of convex functions in $\X$ (parametrized by $\theta$), a scalar $R$ for which $\| f \|_{\infty} \le R$, a sampling distribution $\rho$ supported over $\X$.
\vspace{2mm}
\State{{\bf Random function evaluations.} Draw $2N$ i.i.d. samples according to the distribution $\rho$ and partition them into two sets, $\wh \X = \{\wh \X_1, \wh \X_2 \}$. Generate samples $[f]_{\wh \X_1}$ and $[f]_{\wh \X_2}$ , where $[f]_{\wh\X_i} = \{ f(x): x \in \wh \X_i \}, ~~ i = \{ 1, 2 \}\}$. Denote $[f]_{\wh \X}=\{[f]_{\wh \X_1},[f]_{\wh \X_2}\}$}
\vspace{1mm}
\State{ \textbf{Optimize for $\mu$.}~~Solve the 1D optimization problem 
\begin{equation*} 
\wh \mu =  \underset{\mu\in [-R,R] }{\arg\min} ~  f (\texttt{OPT}(\mu,  [f]_{\wh\X}) ),
\end{equation*} }
\Ensure $\wh x_{\widehat{\mu}} = $\texttt{OPT}($\wh \mu, [f]_{\wh\X}$) .
\end{algorithmic}
\noindent\rule{8.3cm}{0.4pt}
\begin{algorithmic}[1]
\Statex{\textbf{Procedure} \texttt{OPT}($ \mu, [f]_{\wh\X}$)}
  \setcounter{ALG@line}{0}
\State{ \textbf{Estimate the convex envelope.}~~Estimate $\wh f_c=h(\cdot;\wh \theta_\mu)$ by solving Eqn.\ \ref{eq:CORR.core}.}

\State{\textbf{ Optimize the empirical convex envelope.}~~Find an optimizer $\wh x_\mu$ for $\wh f_c$ by solving } 
\begin{equation*} 
\wh x_{\mu}=\underset{x\in \X }{\min} ~ \wh f_c(x),
\end{equation*}
 \Return $\wh x_{\mu}$
\end{algorithmic}
 \caption{Convex Relaxation Regression (CoRR) \label{algo:CORR}}
\end{algorithm} 

To provide further insight into how CoRR works, we point the reader to Fig.\ \ref{fig:demofig}. Here, we show examples of the convex surrogate obtained by \texttt{OPT} for different values of $\mu$. We observe that  as we  vary $\mu$,  the minimum error is attained for $\mu \approx 0.47$. However, when we analytically compute the empirical expectation of convex envelope ($\wh E_2 [ f_c (x)]=0.33$) and use this value for $\mu$, this produces a larger function evaluation. This may seem surprising, as we know that if we set $\mu = \E(f_c(x))$, then the solution of Eqn.~\ref{eq:opt.const1} should provide us the exact convex envelope with the same optimizer as $f$. This discrepancy can be explained by the approximation error introduced through solving the empirical version of Eqn.~\ref{eq:opt.const1}. This figure also highlights the stability of our approach for different values of $\mu$. Our results suggest that our method is robust to the choice of $\mu$, as a wide range of values of $\mu$ produce minimizers close to the true global minimum. Thus CoRR provides an accurate and robust approach for finding the global optimizer of $f$.

\begin{figure}[t!]
\includegraphics[width=1\textwidth]{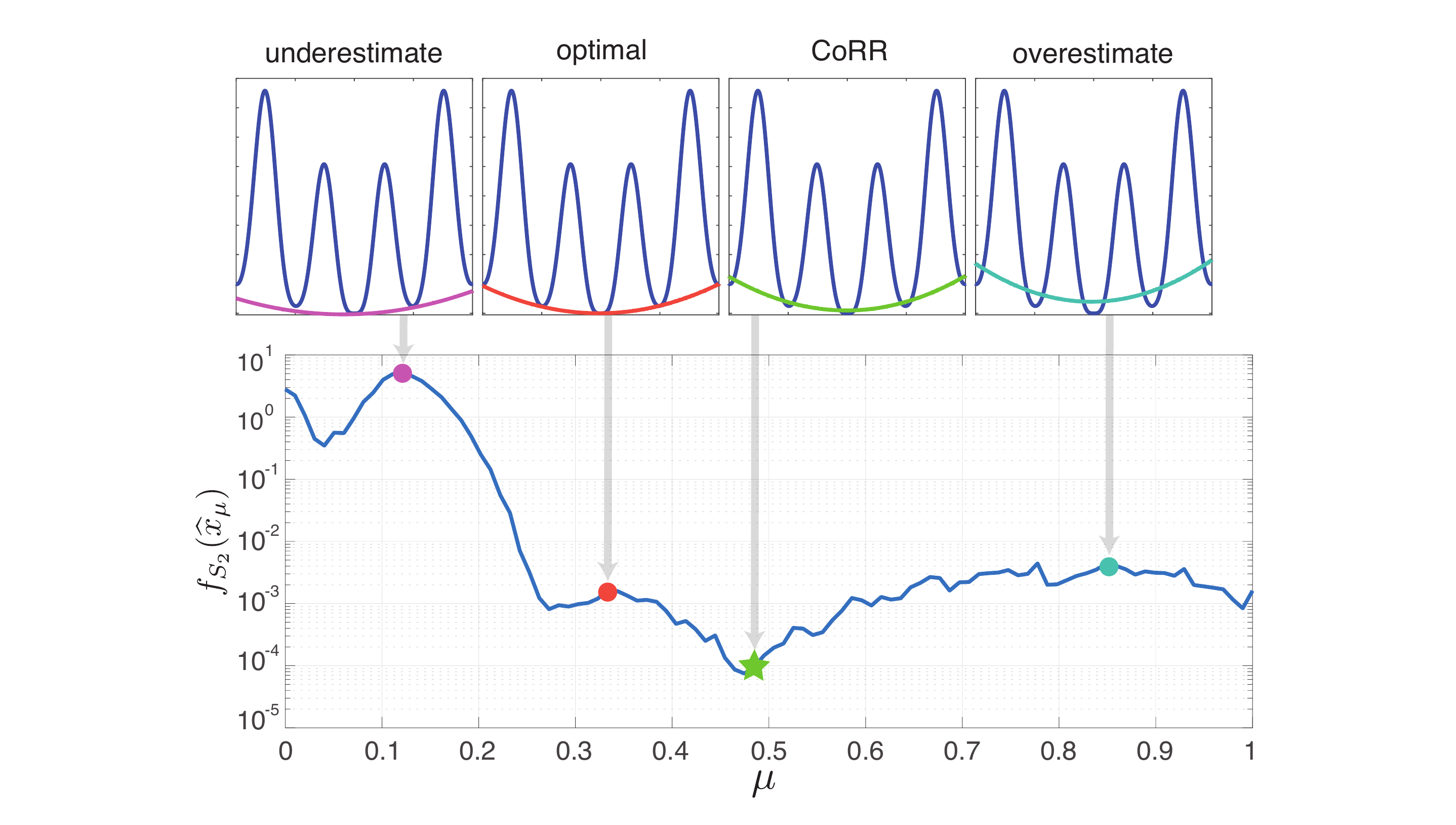}
\vspace{-5mm}
\caption{{\em Estimating the convex envelope of $f$ with CoRR.} Here we demonstrate how CoRR learns a convex envelope by solving Eqn.\ \ref{eq:fmin}. Along the top, we plot the test function $f_{S_2}$ (see Sec.\ \ref{sec:numeric}) and examples of the convex surrogates obtained for different values of $\mu$. From left to right, we display the surrogate obtained for: an underestimate of $\mu$, the empirical estimate of the convex envelope where $\mu \approx \wh{\mathbb{E}}_2[ f_c(x) ]$, the result obtained by CoRR, and an overestimate of $\mu$. Below, we display the value of the function $f_{S_2}$ as we vary $\mu$ (solid blue).}
\label{fig:demofig}
\end{figure}

\section{Theoretical Results}

\label{sec:theory}

In this section, we provide our main theoretical results. We show that as the number of function evaluations $T$  grows, the solution of CoRR converges to the global minimum of $f$ with a polynomial rate. We also discuss the scalability of our result to high-dimensional settings.

\subsection{Assumptions}
We begin by introducing the assumptions required to state our results. The first assumption provides the necessary constraint on the candidate function class $\Hyp$ and the set of all points in $\X$ that are minimizers for the function $f$. 

\begin{asm}[Convexity]
\label{asm:convex} 
Let $\X^*_f$ denote the set of minimizers of $f$. We assume that the following three convexity assumptions hold with regard to every $h(\cdot ; \theta)\in\Hyp$ and $\X^*_f$: (a) $h(x;\theta)$ is a convex function for all $x\in\X$, (b) $h$ is a affine function of $\theta\in\Theta~$ for all $x\in\X$, and (c) $\X^*_f$ is a convex set.
\end{asm}
{\bf Remark.}~Assumption \ref{asm:convex}c does not impose convexity on the function $f$. Rather, it requires  that the set  $\X^*_f$ is convex. This is needed to guarantee that both $f_c$ and $f$ have the same minimizers (see Prop.~\ref{prop:convexenv}). Assumption \ref{asm:convex}c  holds for a large class of non-convex functions. For instance, every continuous function with a unique minimizer satisfies this assumption (see, e.g., our example functions in Sect.~\ref{sec:numeric}). %

Assumption\ \ref{asm:smoothness} establishes the necessary smoothness assumption on the function  $f$ and the function class $\Hyp$.
 
\begin{asm}[Lipschitz continuity]
\label{asm:smoothness}
We assume that $f$ and $h$ are Lipschitz continuous. That is for every $(x_1,x_2)\in \X^2$ we have that $|f(x_1)-f(x_2)| \leq d(x_1,x_2)$. Also for every $x\in \X$ and $(\theta_1,\theta_2)\in \Theta^2$ we have that $|h(x;\theta_1)-h(x;\theta_2)|\leq  U d_2( \theta_1,\theta_2)$. We also assume that every $h\in \Hyp$ is $\lambda$-Lipschitz on $\X$ w.r.t. the metric $d$  for some $\lambda>0$.
\end{asm}

We show that the optimization problem of Eqn.~\ref{prop:convexenv} provides us with the convex envelope  $f_c$ when  the candidate class $\Hyp$ contains $f_c$ (see Lem.~\ref{lem:const.opt}). The following assumption formalizes this condition. 

\begin{asm}[Capacity of  $\Hyp$]
\label{asm:capacity}
We assume that  $f_c \in \Hyp$, that is,  there exist some $h\in\Hyp$  and $\Theta\subseteq\Theta_c$  such that $h(x;\theta)=f_c(x)$ for every $x\in \X$ and $\theta\in \Theta_c$.
\end{asm}

We also require that the following H\"older-type error bounds hold for the distances of our empirical estimates $\widehat x$ and $\widehat \theta$ from  $\X^*_f$ and $\Theta_c$, respectively.

\begin{asm}[H\"older-type error bounds]
\label{asm:eps.opt}
Let  $\Theta^e:=\{\theta|\theta\in \Theta,\E( h(x;\theta) )=f_c(x)\}$.  Also denote $L^*:=\min_{\theta\in\Theta^e}L(\theta)$.  We assume that there exists some finite positive scalars $\gamma_1$,  $\gamma_2$, $\beta_{1}$   and $\beta_2$ such that for every $x\in\X$ and $\theta\in\Theta^e$: (a) $f(x)-f^* \geq\gamma_1 d(x, \X^*_f)^{1/\beta_1} $. (b) $L(\theta) - L^*  \geq  \gamma_2 d_2(\theta,\Theta_c)^{1/\beta_2}$.
\end{asm}


Assumption~\ref{asm:eps.opt} implies that whenever  the error terms $f(x)-f^*$ and $L(\theta)-L^*$  are small, the distances $d(x,\X^*_f)$  and  $d_2(\theta, \Theta_c)$ are small as well. To see why Assumption~\ref{asm:eps.opt} is required for the analysis of CoRR, we note that the  combination of Assumption~\ref{asm:eps.opt} with Assumption~\ref{asm:smoothness} leads to the following \emph{local} bi-H\"older inequalities for every $x\in \X$ and $\theta\in\Theta^e$:
\begin{equation}
\label{eq:biHolder.ineq}
\begin{aligned}
\gamma_1 d(x, \X^*_f)^{1/\beta_1} &\leq f(x)-f^* \leq d(x, \X^*_f)
\\
\gamma_2 d_2(\theta, \Theta_c)^{1/\beta_2} &\leq L(\theta)-L^* \leq U d_2(\theta, \Theta_c)
\end{aligned}
\end{equation}

These inequalities determine the behavior of function $f$ and $L$ around their minimums as they establish upper and lower bounds on the  errors  $f(x)-f^*$ and $L(\theta)-L^*$.  Essentially,  Eqn.\ \ref{eq:biHolder.ineq} implies that  there is a direct relationship between $d(x, \X^*_f)~(d_2(\theta, \Theta_c))$ and  $f(x)-f^*~(L(\theta)-L(\Theta_c))$. Thus,  bounds on $d(x, \X^*_f)$  and $d_2(\theta, \Theta_c)$, respectively, imply bounds on $f(x)-f^*$ and  $L(\theta)-L(\Theta_c)$  and vice versa. These bi-directional bounds are needed due to the fact that CoRR  doest not directly optimize the  function. Instead it optimizes the surrogate loss $L(\theta)$ to find the convex envelope and then it optimizes this  empirical convex envelope to estimate the global minima. This implies that the standard result of optimization theory can only be applied to bound the error $L(\wh\theta)-L^*$.  The inequalities of Eqn.~\ref{eq:biHolder.ineq} are then  required to convert the  bound on $L(\wh\theta)-L^*$ to a bound on  $f(\wh x_{\wh \mu})-f^*$, which ensures that the solution of CoRR converges to a global minimum as $L(\wh\theta)-L^*\to 0$.



It is noteworthy that global error bounds such as those in Assumption~\ref{asm:eps.opt} have been extensively  analyzed in the  literature of approximation theory and variational analysis \citep[see, e.g.,][]{aze2003survey,corvellec2008nonlinear,COV:8137408,Fabian2010}. 
Much of this body of work can be applied to study convex functions such as $L(\theta)$, where one can make use of the basic properties of convex functions to prove lower bounds  on $L(\theta)-L^*$ in terms of the distance between $\theta$ and $\Theta_c$ \citep[see, e.g., Thm. 1.16 in][]{aze2003survey}. While these results are useful to further study the class of functions that satisfy Assumption ~\ref{asm:eps.opt}, providing a direct link between these results and the error bounds of Assumption~\ref{asm:eps.opt} is outside the scope of this paper. 

Assumptions~\ref{asm:capacity}-\ref{asm:eps.opt} can not be applied directly when  $f_c \notin \Hyp$. When $f_c \notin \Hyp$, we make use of the following generalized version of  these assumptions.  We  first consider a relaxed  version of Assumption~\ref{asm:capacity}, which assumes that $f_c$ can be approximated  by some  $h\in\Hyp$.
\begin{asm}[$\upsilon$-approachability of $f_c$ by $\Hyp$]
\label{asm:approach.eps}
Let $\upsilon$ be a positive scalar. Define the distance between the function class $\Hyp$ and $f_c$ as ${\rm dist}(f_c,\Hyp):=\inf_{h\in\Hyp}\E[ |h(x;\theta)-f_c(x)| ],$ where the expectation is taken w.r.t. the distribution $\rho$. 
We  then assume that the following inequality holds:
${\rm dist}(f_c,\Hyp)\leq\upsilon$.
\end{asm}

The next assumption generalizes Assumption \ref{asm:eps.opt}b to the case where $f_c\notin\Hyp$:

\begin{asm}
\label{asm:Holder.eps}
Let $\wt p$ be a positive scalar. Assume that there exists a  class of convex functions $\wt \Hyp\subseteq \C(\X,\R)$  parametrized by $\theta\in \wt \Theta\subset \B(\R^ {\wt p})$ such that: (a) $f_c \in\wt \Hyp$, (b) every $h\in\wt \Hyp$ is linear 
in $\theta$ and (c) $\Hyp\subseteq\wt \Hyp$.  Let $\Theta_c \subseteq \wt \Theta$ be the set of parameters for which $h(x;\theta)=f_c(x)$ for every $x\in \X$ and $\theta\in \Theta_c$.  Also define $\wt \Theta_e:=\{\theta|\theta\in \wt \Theta,\E( h(x;\theta) )=f_c(x)\}$. We assume that there exists some finite positive scalars   $\gamma_2$ and $\beta_2$ such that for every $x\in\X$ and $\theta\in\wt \Theta_e$
 
$$L(\theta) - L^*  \geq  \gamma_2  d_2(\theta, \wt \Theta_c)^{1/\beta_2}.$$

\end{asm}

Intuitively speaking, Assumption~\ref{asm:Holder.eps} implies that the function class $\Hyp$ is a subset of a larger unknown function class $\wt \Hyp$ which satisfies the global error bound of Assumption~\ref{asm:eps.opt}b. Note that we do not require access to the class $\wt \Hyp$, but we need that such a function class exists.

\subsection{Performance Guarantees}
We now present the two main theoretical results of our work and provide sketches of their proofs (the complete proofs of our results is provided in the Supp. Material).
\subsubsection{Exact Setting}
Our first result considers the case where the convex envelope $f_c \in \mathcal{H}$. In this case, we can guarantee that as the number of function evaluations grows, the solution of Alg.\ \ref{algo:CORR} converges to the optimal solution with a polynomial rate.
\begin{theorem}
\label{thm:f.bound.exact}
Let  $\delta$ be a positive scalar. Let Assumptions~\ref{asm:convex}, \ref{asm:smoothness} ,\ref{asm:capacity}, and \ref{asm:eps.opt} hold.  Then  Alg.\ \ref{algo:CORR}  returns $\wh x$ such that with probability $1-\delta$ 
\begin{equation*}
f(\wh x)-f^*= \O\left[\xi_s\left(\frac{\log(1/\delta)}{ T }\right)^{\beta_1\beta_2/2}\right],
\end{equation*}
where the smoothness coefficient $ \xi_s :=(\frac1{\gamma_1})^{\beta_2} (\frac1{\gamma_2})^{\beta_1\beta_2}U^{(1+\beta_2)\beta_1}(RB)^{\beta_2\beta_1}.$
\end{theorem}
{\em Sketch of proof.} To prove this result, we first prove bound on the error  $L(\wh \theta)-\min_{\theta\in\Theta_e} L(\theta)$  for which we rely on standard results from stochastic convex optimization.  This combined with the result of  Lem.~\ref{lem:const.opt}  leads to a bound on  $L(\wh \theta)-L^*$.  The bound on $L(\wh \theta)-L^*$  combined  with Assumption \ref{asm:eps.opt} translates to a bound on $d(\wh x,\X^*_f)$.  The result then follows by applying the Lipschitz continuity assumption (Assumption \ref{asm:smoothness}). \QEDA

Thm.~\ref{thm:f.bound.exact}  guarantees that as the number of function evaluations $T$ grows, the solution of CoRR converges to $f^*$ with a polynomial rate. The order of polynomial depends on the constants $\beta_1$  and $\beta_2$.  The following corollary, which is an immediate result of Thm.~\ref{thm:f.bound.exact}, quantifies the number of function evaluations $T$ needed to achieve an $\eps$-optimal solution. 

 
\begin{corollary}
\label{thm:f.bound.exact}
Let Assumptions~\ref{asm:convex}, \ref{asm:smoothness}, \ref{asm:capacity}, and \ref{asm:eps.opt} hold.  Let $\eps$ and $\delta$ be some positive scalars. Then  Alg.\ \ref{algo:CORR}  needs  $T=(\frac{\xi_s}\eps)^{2/(\beta_1\beta_2)}\log(1/\delta)$ function evaluations to  return $\wh x$ such that  with probability $1-\delta$, $f(\wh x)-f^*\leq \eps$.
\end{corollary}
This result implies that one can achieve an $\eps$-accurate approximation of the global optimizer with CoRR with a polynomial number of function evaluations.


\subsubsection{Approximate Setting}
Thm.\ \ref{thm:f.bound.exact} relies on the assumption that the convex envelope $f_c$  lies in the function class $\Hyp$. However, in general, there is no guarantee that $f_c$ belongs to $\Hyp$. When the convex envelope $f_c \notin \Hyp$, the result of  Thm.~\ref{thm:f.bound.exact} cannot be applied. However, one may expect that Alg.~\ref{algo:CORR} still may find  a close approximation of the global minimum as long as the distance between $f_c$ and $\Hyp$ is small. To prove that CoRR finds a near optimal solution in this case, we must show that $f(\wh x)- f^*$ remains small when the distance between $f_c$ and $\Hyp$ is small. We now generalize Thm.\ \ref{thm:f.bound.exact} to the setting where the convex envelope $f_c$ does not lie in $\Hyp$ but is close to it.
\begin{theorem}
\label{thm:f.bound.approx}
Let Assumptions \ref{asm:convex}, \ref{asm:smoothness}, \ref{asm:approach.eps},  and \ref{asm:Holder.eps} hold.  Then Alg.~\ref{algo:CORR}  returns $\wh x$ such that for every $\zeta>0$ with probability (w.p.) $1-\delta$
\begin{equation*}
f(\wh x)-f^*=   \O\left[\xi_s\left(\sqrt{\frac{\log(1/\delta)}{ T }} +\zeta+\upsilon\right)^{\beta_1\beta_2}\right].
\end{equation*}
\end{theorem}

{\em Sketch of proof.}   To prove this result, we rely on standard results from stochastic convex optimization to first prove a bound on the error $L(\wh \theta)-\min_{\theta\in\Theta^e}L(\theta)$ when we set $\mu$ the empirical mean of the convex envelope. We then make use of Assumption \ref{asm:approach.eps} as well as Lem.~\ref{lem:const.opt} to transform this bound to a bound on $L(\wh \theta)-L^*$.  The bound on $f(\wh x)-f^*$ then follows by combining this result with Assumptions \ref{asm:smoothness} and  \ref{asm:Holder.eps}. \QEDA

 \subsubsection{Approximation Error $\upsilon$ vs. Complexity of $\Hyp$}
\label{sec:approx.nu}
 From function approximation theory, it is known that for a sufficiently smooth function $g$ one can  achieve an 
$\upsilon$-accurate  approximation of $g$ by a linear combination of $p=\O(n/\upsilon)$ bases \citep{mhaskar1996neural,girosi1992convergence}.  Similar \emph{shape preserving} results have been established for the case 
when the function and bases are both convex \citep[see, e.g.,][]{wang2012shape,gal2010shape,shvedov1981orders} under some mild assumptions (e.g., Lipschitz continuity) on $g$.  This implies that to achieve an approximation error of  $\upsilon$ in Thm.~\ref{thm:f.bound.approx},  the function class $\Hyp$ needs  to consist of  $p=\O(n/\upsilon)$ convex bases. Several methods for constructing convex bases that employ the use of polynomial tensors or the kernel methods, are provided in \citet{wang2012shape,gal2010shape}.  Thus to decrease  the approximation error $\upsilon$ one needs to increase the complexity of function class $\Hyp$, i.e., the number of convex bases $p$. 

 \subsubsection{Dependence on Dimension}
 \label{sec:dimension}
 The results of Thm.~\ref{thm:f.bound.exact} and Thm.~\ref{thm:f.bound.approx} have no explicit dependence on the dimension $n$. However, the Lipschitz constant $U$ can, in the worst-case scenario, be of $\O(\sqrt{p})$ (due to the  Cauchy-Schwarz inequality).  On the other hand to achieve an approximation error of $\upsilon$  the number of bases $p$  needs be of $\O(n/\upsilon)$ (see Sect.~\ref{sec:approx.nu}).  When we plug  this result in the bound of Thm.~\ref{thm:f.bound.approx}, this leads to a dependency of $\O(n^{(1+\beta_2)\beta_1/2})$ on the dimension $n$ due to the Lipschitz constant $U$.  In the special case where $\beta_2=\beta_1=1$, i.e., when the error bounds of Assumption \ref{asm:eps.opt} are linear, the dependency on $n$ is also linear. The linear dependency on $n$ in this case matches the results in the literature of black-box (zero-order)  optimization theory \citep[see, e.g.,][]{duchi2015optimal} for the convex case.


\section{Numerical Results}

\label{sec:numeric}
In this section, we evaluate the performance of CoRR on several multi-dimensional test functions used for benchmarking non-convex optimization methods \citep{jamil2013literature}. 

{\bf Evaluation setup.}~~ Here we study CoRR's effectiveness in finding the global minimizer of the following test functions (Fig.\ \ref{fig:phasetrans}a). We assume that all functions are supported over $\X = \mathcal{B}(0,2) \subseteq \R^n$, and otherwise rescale them to lie within this set. {\bf (S1)} Salomon function: $f_S(x) = 1 -\cos(2\pi \|x\|)+ 0.5\|x\|$. {\bf (S2)} Squared Salomon: $f_{S_2}(x) = 0.1f_S(x)^2$. {\bf (SL)} Salomon and Langerman combination:  $f_{SL}(x) = f_S(x) + f_L(x)~ \forall x \in \mathcal{B}(0,10) \cap \mathcal{B}(0,0.2)$ and $f_{SL}(x) = 0$, otherwise (before rescaling the domain). {\bf (L)} Langerman function: $f_L(x) = -\exp(\| x - \alpha \|_2^2/\pi ) \cos( \pi \| x - \alpha \|_2^2 ) +1,~ \forall x \in \mathcal{B}(0,5)$ (before rescaling the domain). {\bf (G)} The Griewank function: $f_G(x) = 0.1 \big[ 1 + \frac{1}{4000} \sum_{i=1}^N x(i)^2 - \prod_{i=1}^N {\frac{\cos(x)}{\sqrt{i}}} \big], ~\forall x \in \mathcal{B}(0,200)$ (before rescaling the domain). All of these functions have their minimum at the origin, except for the Langerman function which has its minimum at $x^{\ast} = c\mathbbm{1}$ for $c = 0.5$. 

All of the aforementioned functions exhibit some amount of global structure for which the convex envelope can be approximated by a quadratic basis (Fig.\ \ref{fig:phasetrans}a). We thus use a quadratic basis to construct our function class $\Hyp$. The basis functions $h(x;\theta) \in \Hyp$ are parameterized by a vector of coefficients $\theta = [ \theta_1, \theta_2, \theta_3 ]$, and can be written as $h(x;\theta) = \langle \theta_1, x^2 \rangle + \langle \theta_2 , x\rangle + \theta_3$. Thus, the number of parameters that we must estimate to find a convex approximation $h$ equals $2n + 1$. In practice, we impose a non-negativity constraint on all entries of the vector $\theta_1$ to ensure that our approximation is convex.

\begin{figure*}[t!]
\begin{center}
\includegraphics[width=0.9\textwidth]{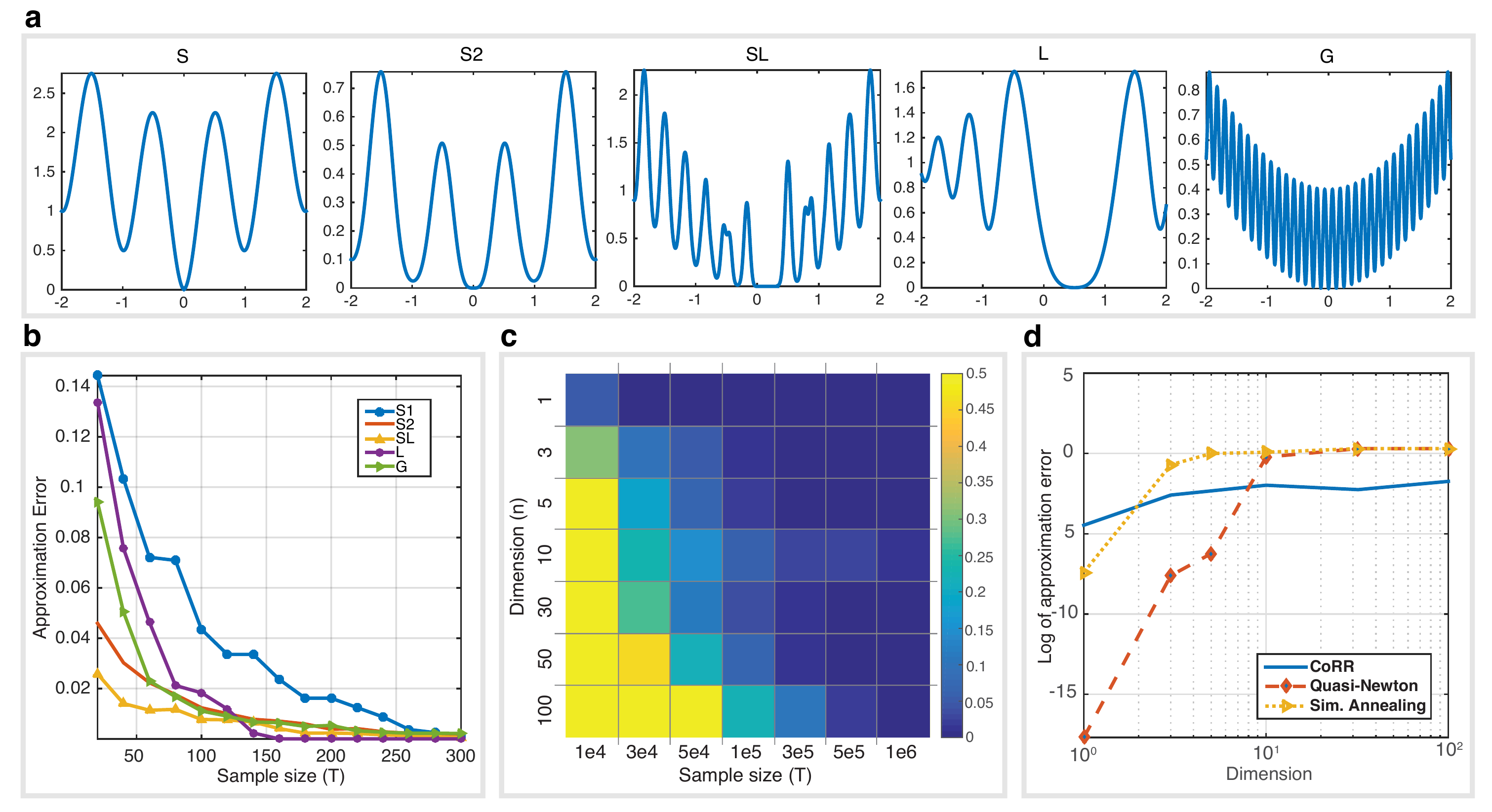}
\end{center}
\caption{{\em Scaling behavior and performance of CoRR.} Along the top row in (a), we plot all five test functions studied in this paper. In (b), we display the mean approximation error between $f(\wh x) - f^*$ as a function of the number of function evaluations $T$ for all test functions in 1D. In (c), we display the mean approximation error as a function of the dimension and number of samples for the Salomon function. In (d), we compare CoRR's approximation error with other approaches for non-convex optimization, as we vary the dimension. \label{fig:phasetrans}}
\end{figure*}

{\bf Summary of results.}~~
To understand the inherent difficulty of finding global minimizers for the test functions above, we compute the error $f(\wh x) - f^{\ast}$ as we increase the number of function evaluations of $f$. Here, we show each of our five test functions (Fig.\ \ref{fig:phasetrans}a) and their average scaling behavior in one dimension (Fig.\ \ref{fig:phasetrans}b), where the error is averaged over $100$ trials. We observe that CoRR quickly converges for all five test functions, with varying convergence rates. We observe the smallest initial error (for only $20$ samples) for $f_{SL}$ and the highest error for $f_S$. In addition, $f_L$ achieves nearly perfect reconstruction of the global minimum after only $200$ samples. The good scaling properties of $f_L$ and $f_{SL}$ is likely due the the fact that both of these functions have a wide basin around their global minimizer. This result provides nice insight into the scaling of CoRR in low dimensions.

Next, we study the approximation error as we vary the sample size and dimension for the Salomon function $f_S$ (Fig.\ \ref{fig:phasetrans}c-d). Just as our theory suggests, there is a clear dependence between the dimension and number of samples required to obtain small error. In Fig.\ \ref{fig:phasetrans}c, we display the scaling behavior of CoRR as a function of both dimension and number of function evaluations $T$. In all of the tested dimensions, we obtain an error smaller than $1e^{-5}$ when we draw one million samples. In Fig.\ \ref{fig:phasetrans}d, we compare the performance of CoRR (for fixed number of evaluations $T$) as we vary the dimension. In contrast, the quasi-Newton (QN) method and hybrid simulated annealing (SA) method \citep{hedar2004heuristic} recover the global minimizer for low dimensions but fail in dimensions greater than ten.\footnote{These methods are selected from a long list of candidates in MATLAB's global optimization toolbox. We report results for the methods that gave the best results for the test functions studied here.} We posit that this is due to the fact the minimizer of the Salomon function lies at the center of its domain and as the dimension of the problem grows, drawing an initialization point (for QN) that is close to the global minimizer becomes extremely difficult.

\section{Discussion and Future Work}

This paper introduced CoRR, an approach for learning a convex relaxation for a wide class of non-convex functions. The idea behind CoRR is to find an empirical estimate of the convex envelope of a function from a set of function evaluations. We demonstrate that CoRR is an efficient strategy for global optimization, both in theory and in practice. In particular, we provide theoretical results (Sec.\ \ref{sec:theory}) which show that CoRR is guaranteed to produce a convergent estimate of the convex envelope that  exhibits polynomial dependence on the dimension. In numerical experiments (Sec.\ \ref{sec:numeric}), we showed that CoRR provides accurate approximations to the global minimizer of multiple test functions and appears to scale well with dimension.


Our current instantiation of CoRR finds a convex surrogate for $f$ based upon a set of samples that are drawn at random at the onset of the algorithm. In our evaluations, we draw i.i.d.\ samples from a uniform distribution over $\X$. However, the choice of the sampling distribution $\rho$ has a significant impact on our estimation procedure. As such, selecting samples in an intelligent manner would significantly reduce the number of samples required to obtain an accurate estimate. A natural extension of CoRR is to the case where we can iteratively refine our distribution $\rho$ based upon the output of the algorithm at previous steps. 


An important factor in the success of our algorithm is the basis that we use to form our approximation. As discussed in Sec.\ \ref{sec:approx.nu}, we know that a polynomial basis can be used to form a convex approximation to any convex function \citep{gal2010shape}. However, finding a concise representation of the convex envelope using high-degree polynomials is not an easy task. Thus finding other well-suited bases  for this approximation, such as the exponential basis, may improve the efficiency of CoRR by reducing the number of bases required. While outside the scope of this paper, exploring the use of constrained dictionary learning methods \citep{yaghoobi2009dictionary} for finding a good basis for our fitting procedure, is an interesting line for future work.

In our experiments, we observe that CoRR typically provides a good approximation to the global minimizer. However, in most cases, we do not obtain machine precision (like QN for low dimensions). Thus, we can combine CoRR with a local search method like QN by using the solution of CoRR as an initialization point for the local search. When using this hybrid approach, we obtain perfect reconstruction of the global minimum for the Salomon function for all of the dimensions we tested (Fig.\ \ref{fig:phasetrans}d). This suggests that, as long the function does not fluctuate too rapidly around its global minimum (Asm.\ \ref{asm:smoothness}), CoRR can be coupled with other local search methods to quickly converge to the absolute global minimizer.

The key innovation behind CoRR is that one can efficiently approximate the convex envelope of a non-convex function by solving a constrained regression problem which balances the approximation error with a constraint on the empirical expectation of the estimated convex surrogate. While our method could be improved by using a smart and adaptive sampling strategy, this paper provides a new way of thinking about how to relax non-convex problems. As such, our approach opens up the possibility of using the myriad of existing tools and solvers for convex optimization problems to efficiently solve non-convex problems. 



\bibliographystyle{apalike}
\bibliography{refs}

\begin{thebibliography}{}

\bibitem[Azar et~al., 2014]{azar2014stochastic}
Azar, M.~G., Lazaric, A., and Brunskill, E. (2014).
\newblock Stochastic optimization of a locally smooth function under correlated
  bandit feedback.
\newblock In {\em ICML}.

\bibitem[Az{\'e}, 2003]{aze2003survey}
Az{\'e}, D. (2003).
\newblock A survey on error bounds for lower semicontinuous functions.
\newblock In {\em ESAIM: ProcS}, volume~13, pages 1--17. EDP Sciences.

\bibitem[Az{\'e} and Corvellec, 2004]{COV:8137408}
Az{\'e}, D. and Corvellec, J.-N. (2004).
\newblock Characterizations of error bounds for lower semicontinuous functions
  on metric spaces.
\newblock {\em ESAIM: Control, Optimisation and Calculus of Variations},
  10:409--425.

\bibitem[B{\"a}ck, 1996]{back1996evolutionary}
B{\"a}ck, T. (1996).
\newblock {\em Evolutionary algorithms in theory and practice: evolution
  strategies, evolutionary programming, genetic algorithms}.
\newblock Oxford University Press.

\bibitem[Bergstra and Bengio, 2012]{bergstra2012random}
Bergstra, J. and Bengio, Y. (2012).
\newblock Random search for hyper-parameter optimization.
\newblock {\em J. Mach. Learn. Res.}, 13(1):281--305.

\bibitem[Bergstra et~al., 2011]{bergstra2011algorithms}
Bergstra, J.~S., Bardenet, R., Bengio, Y., and K{\'e}gl, B. (2011).
\newblock Algorithms for hyper-parameter optimization.
\newblock In {\em NIPS}, pages 2546--2554.

\bibitem[Blake and Zisserman, 1987]{blake1987visual}
Blake, A. and Zisserman, A. (1987).
\newblock {\em Visual reconstruction}, volume~2.
\newblock MIT Press Cambridge.

\bibitem[Boyd and Vandenberghe, 2004]{boyd2004convex}
Boyd, S. and Vandenberghe, L. (2004).
\newblock {\em Convex optimization}.
\newblock Cambridge University Press.

\bibitem[Cand{\`e}s and Tao, 2010]{candes2010power}
Cand{\`e}s, E.~J. and Tao, T. (2010).
\newblock The power of convex relaxation: Near-optimal matrix completion.
\newblock {\em IEEE Trans. Inf. Theory}, 56(5):2053--2080.

\bibitem[Chandrasekaran et~al., 2012]{chandrasekaran2012convex}
Chandrasekaran, V., Recht, B., Parrilo, P.~A., and Willsky, A.~S. (2012).
\newblock The convex geometry of linear inverse problems.
\newblock {\em Found Comput Math}, 12(6):805--849.

\bibitem[Chapelle and Wu, 2010]{chapelle2010gradient}
Chapelle, O. and Wu, M. (2010).
\newblock Gradient descent optimization of smoothed information retrieval
  metrics.
\newblock {\em Inform Retrieval}, 13(3):216--235.

\bibitem[Corvellec and Motreanu, 2008]{corvellec2008nonlinear}
Corvellec, J.-N. and Motreanu, V.~V. (2008).
\newblock Nonlinear error bounds for lower semicontinuous functions on metric
  spaces.
\newblock {\em Math Program}, 114(2):291--319.

\bibitem[Cox, 1958]{cox1958regression}
Cox, D.~R. (1958).
\newblock The regression analysis of binary sequences.
\newblock {\em J R Stat Soc Series B Stat Methodol}, pages 215--242.

\bibitem[Duchi et~al., 2015]{duchi2015optimal}
Duchi, J.~C., Jordan, M.~I., Wainwright, M.~J., and Wibisono, A. (2015).
\newblock Optimal rates for zero-order convex optimization: The power of two
  function evaluations.
\newblock {\em IEEE Trans. Inf. Theory}, 61(5):2788--2806.

\bibitem[Dvijotham et~al., 2014]{Dvijotham14}
Dvijotham, K., Fazel, M., and Todorov, E. (2014).
\newblock Universal convexification via risk-aversion.
\newblock In {\em UAI}, pages 162--171.

\bibitem[Fabian et~al., 2010]{Fabian2010}
Fabian, M.~J., Henrion, R., Kruger, A.~Y., and Outrata, J.~V. (2010).
\newblock Error bounds: Necessary and sufficient conditions.
\newblock {\em Set-Valued and Variational Analysis}, 18(2):121--149.

\bibitem[Falk, 1969]{falk1969lagrange}
Falk, J.~E. (1969).
\newblock Lagrange multipliers and nonconvex programs.
\newblock {\em SIAM J Control Optim}, 7(4):534--545.

\bibitem[Friedman et~al., 2001]{friedman2001elements}
Friedman, J., Hastie, T., and Tibshirani, R. (2001).
\newblock {\em The elements of statistical learning}, volume~1.
\newblock Springer Series in Statistics.

\bibitem[Gal, 2010]{gal2010shape}
Gal, S. (2010).
\newblock {\em Shape-preserving approximation by real and complex polynomials}.
\newblock Springer Science \& Business Media.

\bibitem[Girosi and Anzellotti, 1992]{girosi1992convergence}
Girosi, F. and Anzellotti, G. (1992).
\newblock Convergence rates of approximation by translates.
\newblock Technical report, Massachusetts Inst. of Tech. Cambridge Artificial
  Intelligence Lab.

\bibitem[Grotzinger, 1985]{grotzinger1985supports}
Grotzinger, S.~J. (1985).
\newblock Supports and convex envelopes.
\newblock {\em Math Program}, 31(3):339--347.

\bibitem[Hazan et~al., 2015]{hazan2015graduated}
Hazan, E., Levy, K.~Y., and Shalev-Swartz, S. (2015).
\newblock On graduated optimization for stochastic non-convex problems.
\newblock {\em arXiv:1503.03712 [cs.LG]}.

\bibitem[Hedar and Fukushima, 2004]{hedar2004heuristic}
Hedar, A.-R. and Fukushima, M. (2004).
\newblock Heuristic pattern search and its hybridization with simulated
  annealing for nonlinear global optimization.
\newblock {\em Optim Method and Softw}, 19(3-4):291--308.

\bibitem[Hooke and Jeeves, 1961]{hooke1961direct}
Hooke, R. and Jeeves, T.~A. (1961).
\newblock ``{D}irect search'' solution of numerical and statistical problems.
\newblock {\em J ACM}, 8(2):212--229.

\bibitem[Hutter, 2009]{hutter2009automated}
Hutter, F. (2009).
\newblock Automated configuration of algorithms for solving hard computational
  problems.
\newblock {\em University of British Columbia}.

\bibitem[Jamil and Yang, 2013]{jamil2013literature}
Jamil, M. and Yang, X.-S. (2013).
\newblock A literature survey of benchmark functions for global optimisation
  problems.
\newblock {\em International Journal of Mathematical Modelling and Numerical
  Optimisation}, 4(2):150--194.

\bibitem[Kirkpatrick et~al., 1983]{kirkpatrick1983optimization}
Kirkpatrick, S., Gelatt, C.~D., Vecchi, M.~P., et~al. (1983).
\newblock Optimization by simulated annealing.
\newblock {\em Science}, 220(4598):671--680.

\bibitem[Kleibohm, 1967]{kleibohm1967bemerkungen}
Kleibohm, K. (1967).
\newblock Bemerkungen zum problem der nichtkonvexen programmierung.
\newblock {\em Unternehmensforschung}, 11(1):49--60.

\bibitem[Lewis and Torczon, 1999]{lewis1999pattern}
Lewis, R.~M. and Torczon, V. (1999).
\newblock Pattern search algorithms for bound constrained minimization.
\newblock {\em SIAM J Optimiz}, 9(4):1082--1099.

\bibitem[Mhaskar, 1996]{mhaskar1996neural}
Mhaskar, H. (1996).
\newblock Neural networks for optimal approximation of smooth and analytic
  functions.
\newblock {\em Neural Comput}, 8(1):164--177.

\bibitem[Mobahi and {III}, 2015]{mobahi2015theoretical}
Mobahi, H. and {III}, J. W.~F. (2015).
\newblock A theoretical analysis of the optimization by gaussian continuation.
\newblock In {\em AAAI}.

\bibitem[Mockus et~al., 1978]{mockus1978application}
Mockus, J., Tiesis, V., and Zilinskas, A. (1978).
\newblock The application of bayesian methods for seeking the extremum.
\newblock {\em Towards global optimization}, 2(117-129):2.

\bibitem[Munos, 2011]{munos2011optimistic}
Munos, R. (2011).
\newblock Optimistic optimization of deterministic functions without the
  knowledge of its smoothness.
\newblock In {\em NIPS}.

\bibitem[Munos, 2014]{MAL-038}
Munos, R. (2014).
\newblock From bandits to monte-carlo tree search: The optimistic principle
  applied to optimization and planning.
\newblock {\em Foundations and Trends in Machine Learning}, 7(1):1--129.

\bibitem[Recht et~al., 2010]{recht2010guaranteed}
Recht, B., Fazel, M., and Parrilo, P.~A. (2010).
\newblock Guaranteed minimum-rank solutions of linear matrix equations via
  nuclear norm minimization.
\newblock {\em SIAM Review}, 52(3):471--501.

\bibitem[Rockafellar, 1997]{rockafellar1997convex}
Rockafellar, R.~T. (1997).
\newblock {\em Convex analysis}.
\newblock Princeton University Press.

\bibitem[Shalev-Shwartz and Ben-David, 2014]{shalev2014understanding}
Shalev-Shwartz, S. and Ben-David, S. (2014).
\newblock {\em Understanding machine learning: From theory to algorithms}.
\newblock Cambridge University Press.

\bibitem[Shalev-Shwartz et~al., 2009]{shalev2009stochastic}
Shalev-Shwartz, S., Shamir, O., Srebro, N., and Sridharan, K. (2009).
\newblock Stochastic convex optimization.
\newblock In {\em COLT}.

\bibitem[Shvedov, 1981]{shvedov1981orders}
Shvedov, A.~S. (1981).
\newblock Orders of coapproximation of functions by algebraic polynomials.
\newblock {\em Mathematical Notes}, 29(1):63--70.

\bibitem[Sra et~al., 2012]{sra2012optimization}
Sra, S., Nowozin, S., and Wright, S.~J. (2012).
\newblock {\em Optimization for machine learning}.
\newblock MIT Press.

\bibitem[Tibshirani, 1996]{tibshirani1996regression}
Tibshirani, R. (1996).
\newblock Regression shrinkage and selection via the lasso.
\newblock {\em J R Stat Soc Series B Stat Methodol}, pages 267--288.

\bibitem[Tropp, 2006]{tropp2006algorithms}
Tropp, J.~A. (2006).
\newblock Algorithms for simultaneous sparse approximation. {P}art {II}: Convex
  relaxation.
\newblock {\em Signal Process}, 86(3):589--602.

\bibitem[Wang and Ghosh, 2012]{wang2012shape}
Wang, J. and Ghosh, S.~K. (2012).
\newblock Shape restricted nonparametric regression with bernstein polynomials.
\newblock {\em Computational Statistics \& Data Analysis}, 56(9):2729--2741.

\bibitem[Yaghoobi et~al., 2009]{yaghoobi2009dictionary}
Yaghoobi, M., Blumensath, T., and Davies, M.~E. (2009).
\newblock Dictionary learning for sparse approximations with the majorization
  method.
\newblock {\em IEEE Trans. Signal Process.}, 57(6):2178--2191.

\bibitem[Yuille, 1989]{yuille1989energy}
Yuille, A. (1989).
\newblock Energy functions for early vision and analog networks.
\newblock {\em Biol Cybern}, 61(2):115--123.

\end{thebibliography}

\appendix
\section{Proofs}
\label{sec:anal}

The following result strengthens Proposition~\ref{prop:convexenv} and provides a sufficient condition under which $f$ and its convex envelope $f_c$ have the same set of minimizers. This result implies that one can minimize the function $f$ by minimizing its convex envelope $f_c$, under the assumption that the set of minimizer of $f$, $\X^*_f$, is a convex set.

\begin{lemma}
\label{lem:opt.convex}
Let $f_c$ be the convex envelope of $f$  on $\X$. Let $\X^*_{f_c}$ be the set of minimizers of $f_c$.  Assume that $\X^*_f$ is a convex set. Then $\X^*_{f_c} = \X^*_f$.
\end{lemma}

\begin{proof}
We prove this result by a contradiction argument. Assume that  the result is not true. Then there exists some $\wt x \in \X$ such that $f_c(\wt x)=f^*$ and $\wt x\notin\X^*_f$, i.e., $f(\wt x)>f^*$. By definition of the convex envelope,  $(f^*,\wt x)$ lies in  ${\rm conv}( {\rm epi} f)$. This combined with the fact that ${\rm conv}( {\rm epi} f)$  is the smallest convex set which contains ${\rm epi}f$, implies that 
%
there exists some $z_1=(\xi_1,x_1)$ and $z_2=(\xi_2,x_2)$ in  $\mathrm{epi}f$ and $0\leq \alpha\leq 1$ such that 
\begin{equation}
\label{eq:convex.wtx}
(f^*,\wt x)=\alpha z_1+ (1-\alpha) z_2.
\end{equation}

%


Let us first consider the case in which $z_1$ and $z_2$   belong to the set $\wt {\X^*}=\{ (\xi,x)| x\in \X^*_f, \xi=f(x)\}$. The set  $\wt {\X^*}$ is convex. So every convex combination of its entries also belongs to $\wt {\X^*}$ as well.  This is not the case  for $z_1$ and $z_2$ due to the fact that $(f^*,\wt x)=\alpha z_1+(1-\alpha) z_2$ does not belong to $\wt {\X^*}$ as $\wt x\notin \X^*$.   Now consider the case that either $z_1$ or $z_2$ are not in $\wt {\X^*}$.   Without loss of generality,  assume that $z_1\notin \wt {\X^*}$. In this case,  $\xi_1$  must be larger than $f^*$ since $x_1\notin \X^*_f$. This implies that $(f^*,\wt x)$ can not be expressed as the convex combination of $z_1$ and $z_2$ since in this case: \textbf{(i)} for every $0<\alpha\leq 1$, we have that $\alpha \xi_1+ (1-\alpha) \xi_2> f^*$ and \textbf{(ii)} when $\alpha=0$, then $x_2=\wt x$ and therefore $\alpha \xi_1+ (1-\alpha) \xi_2=\xi_2=f(\wt x)> f^*$.  Therefore Eqn.~\ref{eq:convex.wtx} can not hold for any $z_1,z_2\in\mathrm{epi}f$  when $0\leq \alpha\leq 1$. Thus the assumption that there exists some $\wt x \in \X/\X^*_f$ such that $f_c(\wt x)=f^*$ can not be true either, which proves the result. 
\end{proof}

\subsection{Proof of Lem.~\ref{lem:const.opt}}
\label{proof:lemma1}

We first prove that any underestimate (lower bound) of function $f$  (except $f_c$) does not satisfy the constraint of  the optimization problem of Eqn.~\ref{eq:opt.const1}. This is due to the fact that for any underestimate $h(\cdot;\theta)\in\Hyp/f_c$, there exists  some $x_u \in \X$ and $\eps>0$ such that for every $\theta_c\in\Theta_c$
\begin{equation*}
\begin{aligned}
|h(x_u;\theta)-h(x_u;\theta_c)| &= h(x_u;\theta_c)-h(x_u;\theta)\\
&=f_c(x_u)-h(x_u;\theta)=\eps.
\end{aligned}
\end{equation*}

For every $x\in\X$, the following then holds due to the fact that the function class $\Hyp$ is assumed to be Lipschitz:
\begin{equation}
\label{eq:h.lip}
\begin{aligned}
&h(x;\theta)-h(x;\theta_c) =  h(x;\theta) -h(x_u, \theta )  - \eps \\
&h(x_u, \theta_c ) - h(x;\theta_c) \leq 2\lambda d(x,x_u)-\eps.
\end{aligned}
\end{equation}

Eqn.~\ref{eq:h.lip} implies that for every $x\in \B(x_u,\eps/ 2\lambda)$   the  inequality $\Delta_c(x) = h(x;\theta_c) -h(x; \theta)>0$ holds. Denote the event $\{x\in\B(x_u,\eps/(2\lambda))\}$ by $\Omega_u$. 
We then deduce that
\begin{align*}
&\E [ \Delta_c(x) ] \geq \P (\Omega_u) \E [ \Delta_c(x) | \Omega_u ] > 0,
\end{align*}
where the last inequality follows due to the fact that  both $\P( \Omega_u )$ and  $\E [ \Delta_c(x)  |  \Omega_u ]$ are larger than $0$.  The inequality $\P( \Omega_u )>0$ holds since $\rho(x)>0$ for every $x\in \X$ and also that $\B(x_u,\eps/2\lambda)\neq\emptyset$. The inequality $\E [  \Delta_c(x) | \Omega_u ] >0$  holds by the fact that for every $x\in \B(x_u,\eps/2\lambda)$ the inequality $ \Delta_c(x) >0$ holds.

Let   $\wt \Hyp:=\{ h:h \in \Hyp, \E[ h(x;\theta)] =\E [ f_c(x) ] \}$ be  a set of all functions $h$ in $\Hyp$ with the same mean as the convex envelope $f_c$. We now show that $f_c$ is the only minimizer of  $L(\theta)=\E[ |h(x;\theta)-f(x)| ]$ that lies in the set $\wt \Hyp$. We do this by proving that for every $h \in \wt \Hyp/f_c$, the  loss $L(\theta) > L(\theta_c)$, for every $\theta_c\in \Theta_c$. First we recall that any underestimate $h\in\Hyp/f_c$ of $f$ can not lie in $ \wt \Hyp$, as we have already shown that $\E[ h(x;\theta) ] <\E [ f_c(x) ] $ for every $h\in\Hyp/f_c$. This implies that for every $h\in\wt \Hyp/f_c$ there exists some $x_o\in \X$ such that $h(x_o;\theta)>f(x)$, or equivalently, we have that   for every $h\in\wt \Hyp/f_c$ there exists  some $x_o \in \X$ and $\eps>0$ such that 
\begin{align*}
|h(x_o;\theta)-f(x_o)|&=h(x_o;\theta)-f(x_o)=\eps.
\end{align*}

Then for every $x\in\X$, the following  holds due to the fact that the function class $\Hyp$ and $f$ are assumed to be Lipschitz:
\begin{align}
\label{eq:hf.lip}
&h(x;\theta)-f(x) = h(x;\theta)-h(x_o, \theta ) +\eps\\
& f(x_o ) - f(x) \geq -2\lambda d(x,x_o).
\end{align}
 Eqn.~\ref{eq:hf.lip} implies that for every $x\in \B(x_o,\eps/2\lambda)$   the  inequality $h(x; \theta)-f_c(x) >0$ holds. Denote the event $\{x\in\B(x_o,\eps/2\lambda)\}$ by $\Omega_o$. Let $\Delta(x) = f(x) - h(x;\theta)$. We then deduce
\begin{align}
\notag & \E[ |h(x;\theta)-f(x)| ] \\
\label{eq:opt.const11}& = \P(\Omega_o ) \E[ | \Delta(x) | ~|~ \Omega_o ] + \P(\Omega_o^c) \E[ |\Delta(x)| ~|~ \Omega^c_o ]\\
\label{eq:opt.const12}&>\P(\Omega_o) \E[ \Delta(x) ~| \Omega_o ] + \P(\Omega^c_o) \E[ \Delta(x) ~| \Omega^c_o ] \\
 \label{eq:opt.const13} &=\E[ \Delta(x) ]= \E[ f(x)-f_c(x) ].
\end{align}
Line~\eqref{eq:opt.const11} holds by the law of total expectation.  The inequality~\eqref{eq:opt.const12} holds since   $h(x;\theta)>f(x)$  for every $x\in \B(x_o,\eps/2\lambda)$. This implies that $|h(x;\theta)-f(x)|>0>f(x)-h(x;\theta)$. Line \eqref{eq:opt.const13} holds since $\E[  h(x;\theta) ] = \E[ f_c(x) ]$ for $h\in\wt \Hyp$.  The fact that $L(\theta)=\E[ |h(x;\theta)-f(x)| ] >\E[ | f(x)-f_c(x) | ] = L(\theta_c)$ for every $h(\cdot;\theta)\in\Hyp/f_c$ implies that  the set of minimizers of  $L(\theta)$ coincide with the set $\Theta_c$, which completes the proof.

\subsection{Proof of  Thm.~\ref{thm:f.bound.exact}}

To prove the result of Thm.~\ref{thm:f.bound.exact}, we need to relate the solution of the optimization problem of Eqn.~\ref{eq:CORR.core} with the result of  Alg.~\ref{algo:CORR}, for which we rely on the following lemmas.

Before we proceed, we must introduce some new notation. Define the convex sets $\Theta^e$   and $\wh \Theta^e$ as $\Theta^e:=\{\theta:\theta\in\Theta, \E[h(x;\theta)]= \E[f_c(x)]\}$  and $\wh \Theta^e:=\{\theta: \theta\in\Theta, \wh \E_2 [h(x;\theta)]=\wh \E_2[f_c(x)] \}$, respectively. Also define the subspace $\Theta_{\text{sub}}:=\{\theta:\theta\in\R^p,\E[h(x;\theta)]=\E[f_c(x)]\}$.

\begin{lemma}
\label{lem:Lalp.tail}
Let $\delta$ be a positive scalar. Under Assumptions \ref{asm:convex} and \ref{asm:capacity} there exists some $\mu\in[-R,R]$ such that the following holds w.p. $1-\delta$:
\begin{equation*}
\big| L(\wh \theta_\mu)-\min_{\theta\in\Theta^e}L(\theta)\big| \leq \O\left( B R U\sqrt{ \frac{\log (1/\delta) }{T}}\right).
\end{equation*}

\end{lemma}

\begin{proof}
The empirical estimate $\wh \theta_\mu$ is obtained by minimizing the empirical $\wh L(\theta)$ under some affine  constraints. Additionally, the function $ L(\theta)$ takes the form of the expected value of a generalized linear model.  Now set $\mu=\wh \E_2 [ f_c(x) ]$.  In this case, the following result on stochastic optimization of the generalized linear model holds  for $\mu=\wh \E_2 [ f_c(x) ]$ w.p. $1-\delta$ \citep[see, e.g.,][for the proof] {shalev2009stochastic}:
\begin{equation*}
L(\wh \theta_\mu)-\min_{\theta\in \wh \Theta^e} L(\theta)= \O\left( B R U_1\sqrt{ \frac{\log (1/\delta) }{T}}\right),
\end{equation*}
where  $U_1$ is the Lipschitz constant of $| h(x;\theta) - f(x) |$. We then deduce that for every $x\in \X$,  $\theta\in\Theta$ and $\theta'\in\Theta$,
\begin{equation*}
\left|~|h(x,\theta)-f(x)|-|h(x,\theta')-f(x)|~\right|\leq U_1 \|\theta-\theta'\|.
\end{equation*}
The inequality $\left|~|a|-|b|~\right|\leq |a-b|$, combined with the fact that for  every $x\in \X$ the function  $h(x;\theta)$ is Lipschitz continuous in $\theta$ implies,
\begin{equation*}
\begin{aligned}
&\left|~|h(x,\theta)-f(x)|-|h(x,\theta')-f(x)|~\right| 
\\
\leq& |h(x,\theta)-h(x,\theta')|\leq U \|\theta-\theta'\|.
\end{aligned}
\end{equation*}

Therefore the following holds:
\begin{equation}
\label{eq:empric.mininimze.tail}
L(\wh \theta_\mu)-\min_{\theta\in \wh \Theta^e} L(\theta)   = \O\left( B R U\sqrt{ \frac{\log (1/\delta) }{T}}\right).
\end{equation}

For every $\theta\in\wh\Theta^e$,   the following holds w.p. $1-\delta$:
\begin{equation*}
\begin{aligned}
\E [ h(x;\theta) ] -\wh \E_2 [ f_c(x) ]  &=\E [ h(x;\theta) ]  -\wh \E_2 [ h(x;\theta) ]\\ 
&\leq R\sqrt{\frac {\log(1/\delta)}{2T}},
\end{aligned}
\end{equation*}
as well as,
\begin{equation*}
\begin{aligned}
\wh \E_2 [ f_c(x) ] -\E [ f_c(x) ]  \leq R\sqrt{\frac {\log(1/\delta)}{2T}},
  \end{aligned} 
\end{equation*}
in which we rely on the H\"oeffding inequality for concentration of measure. These results combined with a union bound argument  implies that:
 
 \begin{equation}
\label{eq:expect.diff.tail}
\begin{aligned}
\E [ h(x;\theta) ] -\E [ f_c(x) ] &=\E [ h(x;\theta) ] -\wh \E_2 [ f_c(x) ]\\
&+\wh \E_2 [ f_c(x) ] -\E [ f_c(x) ]\\
&\leq R\sqrt{\frac {2\log(2/\delta)}{T}},
\end{aligned}
\end{equation}
for every $\theta\in\wh \Theta^e$.  We know that  $\min_{\theta\in \wh\Theta^e}L(\theta) \leq L( \theta_c)$, due the fact that $\theta_c\in\wh\Theta^e$. This combined with the fact  that $\theta_c=\min_{\theta\in \Theta^e}L(\theta)$ leads to the following sequence of inequalities w.p. $1-\delta$:
\begin{align*}
&\min_{\theta\in\wh \Theta^e}L(\theta)\leq L(\theta_c)=\E[ f(x)-f_c(x) ]
\\
&\leq\E[ |f(x)-h(x;\wh \theta_c)| ] +\E [ h(x;\wh \theta_c)-f_c(x) ]
\\
&\leq \min_{\theta\in \wh\Theta^e}L(\theta)+R\sqrt{\frac {2\log(2/\delta)}{T}},
\end{align*}
where the last inequality follows from the bound of Eqn.~\ref{eq:expect.diff.tail}. It immediately follows that:
\begin{align*}
\Big|\min_{\theta\in \wh\Theta^e}L(\theta)- \min_{\theta\in \Theta^e}L(\theta)\Big|\leq R\sqrt{\frac {2\log(2/\delta)}{T}},
\end{align*}
w.p. $1-\delta$. This combined with Eqn.~\ref{eq:empric.mininimze.tail} completes the proof.
\end{proof}

Let $\wh \theta^{\rm proj}_{\mu}$ be the $\ell_2$-normed projection of $\wh \theta_{\mu}$ on the subspace $\Theta_{\text{sub}}$.  We now prove bound on the error  $\|\wh \theta^{\rm proj}_\mu-\wh \theta_\mu \|$.

\begin{lemma}
\label{lem:thdiff.tail}
Let $\delta$ be a positive scalar. Then under Assumptions \ref{asm:convex} and \ref{asm:capacity} there exists  some $\mu\in[-R,R]$ such that the following holds with probability $1-\delta$:

\begin{equation*}
\|\wh \theta^{\rm proj}_\mu-\wh \theta_\mu \| \leq \frac R{\|\E [ \phi(x) ] \|}\sqrt{ \frac{2\log (4/\delta) }{T}}.
\end{equation*}

\end{lemma}

\begin{proof}

 Set $\mu=\mu_f:=\E [ f_c(x) ]$. Then $\wh \theta^{\rm proj}_\mu$ can be obtained as  the solution of following optimization problem:
\begin{equation*}
\wh \theta^{\rm proj}_\mu=\underset{\theta\in \R^p}{\arg\min} \|\theta-\wh \theta_{\mu}\|^2\qquad\text{s.t.}\qquad\E [ h(x;\theta) ] =\mu_f.
\end{equation*}
Thus $\wh \theta^{\rm proj}_\mu$ can be obtain as the extremum of the following Lagrangian:

\begin{equation*}
\L(\theta,\lambda)=\|\theta-\wh \theta_\mu\|^2+\lambda(\E [ h(x;\theta) ] -\mu_f).
\end{equation*}

This problem can be solved in closed-form as follows:

\begin{equation}
\label{eq:lagrange}
\begin{aligned}
0&=\frac{\partial \L(\theta,\lambda)}{\partial \theta}=\theta-\wh \theta_\mu+\lambda \E [ \phi(x) ]
\\
0&=\frac{\partial \L(\theta,\lambda)}{\partial \lambda}= \E [ h(x;\theta) ] -\mu_f.
\end{aligned}
\end{equation}

Solving the  above system of equations leads to $\E [ h(x; (\wh \theta_\mu-\lambda \E [ \phi(x) ]) ]=\mu_f$. The solution for $\lambda$ can be obtained as
\begin{equation*} 
\lambda=\frac{\mu_f-\E [ h(x;\wh \theta_\mu)]}{\|\E[ \phi(x) ] \|^2}.
\end{equation*}

By plugging this in Eqn.~\ref{eq:lagrange} we deduce:

\begin{equation*} 
\wh \theta^{\rm proj}_\mu=\wh \theta_{\mu}-\frac{(\mu_f-\E [ h(x;\wh \theta_\mu) ] )\E [ \phi(x) ]}{\|\E [ \phi(x) ] \|^2}, 
\end{equation*}

For the choice of  $\mu=\wh \E_2 [ f_c(x) ]$ we deduce:
\begin{equation*} 
\begin{aligned}
\|\wh \theta^{\rm proj}_\mu-\wh \theta_{\mu}\|&=\frac{| \mu_f-\E [ h(x;\wh \theta_\mu)] |}{\|\E [ \phi(x) ]\|}
\\
&=\frac{|\E [ f_c(x) ]  -\E [ h(x;\wh \theta_\mu)] |}{\|\E [ \phi(x) ] \|}.
\end{aligned}
\end{equation*}

This combined with Eqn.~\ref{eq:expect.diff.tail} and a union bound proves the result.
\end{proof}

We proceed by proving bound on the absolute error   $|L(\wh \theta^{\rm proj}_\mu)-L(\theta_c)|=|L(\wh \theta^{\rm proj}_\mu)-\min_{\theta\in\Theta^e}L(\theta)|$.

\begin{lemma}
\label{lem:Lwt.tail}
Let $\delta$ be a positive scalar. Under Assumptions \ref{asm:convex} and \ref{asm:capacity} there exists some $\mu\in[-R,R]$ such that the following holds with probability $1-\delta$:
\begin{equation*}
\big| L(\wh \theta^{\rm proj}_\mu)-L(\theta_c)\big| = \O\left( B R U\sqrt{ \frac{\log (1/\delta) }{T}}\right).
\end{equation*}
\end{lemma}

\begin{proof} 
From Lem.~\ref{lem:thdiff.tail} we deduce:

\begin{equation}
\label{eq:diff.h}
\begin{aligned} 
&|\E[ h(x;\wh \theta^{\rm proj}_\mu) -h(x;\wh \theta_\mu) ] | \\
&\leq \|\wh \theta^{\rm proj}_\mu-\wh \theta_{\mu}\|\|\E[ \phi(x) ]\| \leq 2R\sqrt{\frac {\log(4/\delta)}{T}},
\end{aligned}
\end{equation}
where the first inequality is due to  the Cauchy-Schwarz inequality. We then deduce:
\begin{align*} 
&|~|L(\wh\theta^{\rm proj}_{\mu})-L(\theta_c)|-|L(\wh\theta_{\mu})-L(\theta_c)|~|\\
&\leq |L(\wh\theta^{\rm proj}_{\mu})-L(\wh\theta_{\mu})|\leq |\E [ h(x;\wh \theta^{\rm proj}_\mu) -h(x;\wh \theta_\mu) ]|,
\end{align*}
in which we rely on the triangle inequality $|~|a|-|b|~|\leq |a-b|$. It then follows that

\begin{equation*}
\begin{aligned}
L(\wt\theta_{\mu})-L(\theta_c)&\leq |L(\wh\theta_{\mu})-L(\theta_c)|
\\
&+ |\E [ h(x;\wh \theta^{\rm proj}_\mu) -h(x;\wh \theta_\mu) ]|.
\end{aligned}
\end{equation*}

Combining this result with the result of Lem.~\ref{lem:Lalp.tail} and Eqn.~\ref{eq:diff.h} proves the result.

\end{proof}

In the  following lemma we make use of Lem.~\ref{lem:thdiff.tail} and Lem.~\ref{lem:Lwt.tail} to prove that   the minimizer  $\wh x_\mu={\arg\min}_{x\in\X} h(x;\wh \theta_{\mu})$ is close to a global minimizer $x^*\in \X^*_f$.

\begin{lemma}
\label{lem:x.bound}
Under Assumptions  \ref{asm:convex}, \ref{asm:capacity} and  \ref{asm:eps.opt} there exists some  $\mu\in[-R,R]$  such that w.p. $1-\delta$:

\begin{align*}
d( \wh x_\mu,\X^*_f) = \O \left( \left(\frac{\log (1/\delta) }{T}\right)^{\beta_1\beta_2/2}\right).
\end{align*}

\end{lemma}

\begin{proof}
The result of Lem.~\ref{lem:Lwt.tail} combined with Assumption~\ref{asm:eps.opt}.b implies that w.p. $1-\delta$:

\begin{equation*}
d_2(\wh \theta^{\rm proj}_\mu, \Theta_c) \leq \left(\frac{\eps_1(\delta)}{\gamma}\right)^{\beta_2},
\end{equation*}
where $\eps_1(\delta)=B R U\sqrt{ \frac{\log (1/\delta) }{T}}$. This combined with the result of Lem.~\ref{lem:thdiff.tail} implies that w.p. $1-\delta$:
\begin{equation*}
d_2(\wh\theta_\mu, \Theta_c )\leq d_2(\wh \theta^{\rm proj}_\mu,\Theta_c) +d_2(\wh \theta^{\rm proj}_\mu, \wh \theta_\mu) \leq 2\left( \frac{\eps_c(\delta)} {\gamma_2}\right)^{\beta_2},
\end{equation*}
where $\eps_c(\delta)=\O\left( \frac{RB U}{ \min (1,\|\E[ \phi(x)]\|)} \sqrt{\frac{\log \frac 1\delta}{ T}} \right)$.

We now use this result to prove a high probability bound on $f_c(\wh x_\mu)-f^*$ :
\begin{equation*}
\begin{aligned}
&f_c(\wh x_\mu)-f^*=h(\theta_c,\wh  x_{\mu})- h(\theta_c, x^*) \\
&=h(\theta_c,\wh  x_\mu)- h(\wh \theta_\mu, \wh x_\mu)+  \min_{x\in\X} h(\wh \theta_\mu, x)- h(\theta_c,  x^*)\\
&\leq h(\theta_c,\wh  x_\mu)- h(\wh \theta_\mu, \wh x_\mu)+  h(\wh \theta_\mu,x^*)- h(\theta_c,  x^*)\\
&\leq 2U d_2( \wh \theta_{\mu},\Theta_c)\leq 2 U\left(\frac{ \eps_c(\delta)}{ \gamma_2}\right)^{\beta_2},
\end{aligned}
\end{equation*}
where the last inequality follows by  the fact that $h$ is U-Lipschitz w.r.t. $\theta$. This combined with Assumption~\ref{asm:eps.opt}.a completes the proof.

\end{proof}

It then follows by combining the result of  Lem.~\ref{lem:x.bound}, Assumption~\ref{asm:smoothness} and  the fact that $f_c$ is the tightest convex lower bound of function $f$   that there exist a $ \mu =[-R, R]$ such that

\begin{equation*}
f(\wh x_\mu)-f^*=\O\left[ \left(\frac{\log (1/\delta) }{T}\right)^{\beta_1\beta_2/2}\right]
\end{equation*}

This combined with the fact that  $f(\wh x_{\wh\mu})\leq f(\wh x_{\mu})$ for every $\mu \in[-R,R]$, completes the proof of the main result  (Thm.~\ref{thm:f.bound.exact}) .

\subsection{Proof of Thm.~\ref{thm:f.bound.approx}}

We prove this theorem by generalizing the result of Lems.~\ref{lem:Lalp.tail}-\ref{lem:x.bound} to the case that  $f\notin \Hyp$. First we need to introduce some notation. Under the assumptions of Thm.~\ref{thm:f.bound.approx}, for every $\zeta>0$, there exists some  $\theta^{\zeta}\in \Theta$ and $\upsilon>0$ such that  the following inequality holds:

\begin{equation*}
 \E[ |h(x;\theta^{\zeta})-f_c(x)| ]\leq\upsilon+\zeta.
 \end{equation*}
 
 Define the convex sets $\wt \Theta^\zeta:=\{\theta: \theta\in\Theta,  \E_2 [ h(x;\theta) ] = \E_2 [ h(x;\theta^\zeta) ] \}$  and $\wh \Theta^\zeta:=\{\theta: \theta\in\Theta, \wh \E_2 [ h(x;\theta) ] =\wh \E_2 [ h(x;\theta^\zeta) ] \}$. Also define the subspace $\Theta^\zeta_{\text{sub}}:=\{\theta:\theta\in\R^{\wt p},\E [ h(x;\theta) ] =\E[ h(x;\theta^{\zeta}) ] \}$.

\begin{lemma}
\label{lem:Lalp.tail.approx}
Let $\delta$ be a positive scalar. Under Assumptions \ref{asm:convex} and \ref{asm:approach.eps} there exists some $\mu\in[-R,R]$ such that  for every $\zeta>0$ the following holds with probability $1-\delta$:
\begin{equation*}
\big| L(\wh \theta_\mu)-\min_{\theta\in\wt \Theta^\zeta}L(\theta)\big| = \O\left( B R U\sqrt{ \frac{\log (1/\delta) }{T}}\right)+\upsilon+\zeta.
\end{equation*}

\end{lemma}

\begin{proof} 
The empirical estimate $\wh \theta_\mu$ is obtained by minimizing the empirical $\wh L(\theta)$ under some affine  constraints. Also the function $ L(\theta)$ is in the form of expected value of some generalized linear model.  Now set $\mu=\wh \E_2 [ h(x;\theta^\zeta) ]$. Then the following result on stochastic optimization of the generalized linear model holds  w.p. $1-\delta$ \citep[see, e.g.,][for the proof]{shalev2009stochastic}:
\begin{equation*}
L(\wh \theta_\mu)-\min_{\theta\in \wh \Theta^\zeta} L(\theta)= \O\left( B R U_1\sqrt{ \frac{\log (1/\delta) }{T}}\right),
\end{equation*}
where  $U_1$ satisfies the following Lipschitz continuity  inequality for every $x\in \X$,  $\theta\in\Theta$ and $\theta'\in\Theta$:
\begin{equation*}
\left|~|h(x,\theta)-f(x)|-|h(x,\theta')-f(x)|~\right|\leq U_1 \|\theta-\theta'\|.
\end{equation*}
The inequality $\left|~|a|-|b|~\right|\leq |a-b|$  combined with the fact that for  every $x\in \X$ the function  $h(x;\theta)$ is Lipschitz continuous in $\theta$ implies  
\begin{equation*}
\begin{aligned}
&\left|~|h(x,\theta)-f(x)|-|h(x,\theta')-f(x)|~\right| 
\\
\leq& |h(x,\theta)-h(x,\theta')|\leq U \|\theta-\theta'\|.
\end{aligned}
\end{equation*}

Therefore the following holds:
\begin{equation}
\label{eq:empric.mininimze.tail.approx}
L(\wh \theta_\mu)-\min_{\theta\in \wh \Theta^\zeta} L(\theta)   = \O\left( B R U\sqrt{ \frac{\log (1/\delta) }{T}}\right),
\end{equation}

For every $\theta\in\wh\Theta^\zeta$   the following holds w.p. $1-\delta$:
\begin{equation*}
\begin{aligned}
&\E [ h(x;\theta) ] -\wh \E_2 [ h(x;\theta^{\zeta}) ] =\E [ h(x;\theta)] -\wh \E_2 [ h(x;\theta)]
\\ 
&\leq R\sqrt{\frac {\log(1/\delta)}{2T}},
\end{aligned}
\end{equation*}
as well as,
\begin{equation*}
\begin{aligned}
\wh \E_2 [ h(x;\theta^{\zeta}) ] -\E [ h(x;\theta^{\zeta}) ]  \leq R\sqrt{\frac {\log(1/\delta)}{2T}},
  \end{aligned} 
\end{equation*}
in which we rely on the H\"oeffding inequality for concentration of measure. These results combined with a union bound argument implies that
 \begin{equation}
 \begin{aligned}
\label{eq:expect.diff.tail.approx}
&\E [ h(x;\theta) ] -\E [ h(x;\theta^{\zeta}) ] =\E [ h(x;\theta)] -\wh \E_2 [ h(x;\theta^{\zeta}) ]\\
&+\wh \E_2 [ h(x;\theta^{\zeta}) ] -\E [ h(x;\theta^{\zeta})]
\leq R\sqrt{\frac {2\log(2/\delta)}{T}},
\end{aligned}
\end{equation}
for every $\theta\in\wh \Theta^\zeta$. Then the following sequence of inequalities holds:
%

\begin{equation*}
\begin{aligned}
&\min_{\theta\in \wh\Theta^\zeta}L(\theta) \leq L( \theta^\zeta)=\E[ |h(x;\theta^{\zeta})- f(x) |]
\\ 
\leq&L(\theta_c)+\E[ |h(x; \theta^\zeta)-f_c(x)| ]
\\
\leq&L(\theta_c)+\upsilon+\zeta
\\
\leq& \min_{\theta\in \wh\Theta^\zeta}L(\theta)+R\sqrt{\frac {2\log(2/\delta)}{T}}.
\end{aligned}
\end{equation*}

The first inequality follows from the fact that $\theta_c\in\wh\Theta^\zeta$. Also the following holds w.p. $1-\delta$:

\begin{equation*}
\begin{aligned}
&L(\theta_c)\leq \E[ |h(x;\theta^{\zeta})- f_c(x) | ] +\E[ h(x;\theta^{\zeta}) ] -\E[ f(x) ]
\\ 
\leq&\upsilon+\zeta+\E[ h(x;\theta^{\zeta}) ] -\E [ f(x) ]
\\
\leq&\min_{\theta\in \wh\Theta^\zeta}\E [ h(x;\theta) ] -\E [ f(x) ] +R\sqrt{\frac {2\log(2/\delta)}{T}}+\upsilon+\zeta
\\
\leq& \min_{\theta\in \wh\Theta^\zeta}L(\theta)+R\sqrt{\frac {2\log(2/\delta)}{T}}+\upsilon+\zeta.
\end{aligned}
\end{equation*}

 The last inequality follows from the bound of Eqn.~\ref{eq:expect.diff.tail.approx}. It immediately follows that
\begin{equation*}
\begin{aligned}
\Big|\min_{\theta\in \wh\Theta^\zeta}L(\theta)- \min_{\theta\in\Theta^e}L(\theta)\Big|\leq R\sqrt{\frac {2\log(2/\delta)}{T}}+\upsilon+\zeta,
\end{aligned}
\end{equation*}
w.p. $1-\delta$. This combined with Eqn.~\ref{eq:empric.mininimze.tail.approx} completes the proof.

\end{proof}

Under Assumption \ref{asm:Holder.eps}, for every  $ h(\cdot; \theta)\in \Hyp$, there exists some $h(\cdot; \wt \theta)\in \wt \Hyp$  such that     $h(x;\theta)=h(x; \wt \theta)$ for every $x\in\X$. Let $\wt \theta_\mu$ be the corresponding set of parameters for $\wh \theta_\mu$ in $\wt \Theta$.  Let $\wt \theta^{\rm proj}_\mu$ be the $\ell_2$-normed projection of $\wt \theta_\mu$ on the subspace $\Theta^\zeta_{\text{sub}}$.  We now prove bound on the error  $\|\wt \theta_\mu-\wt \theta^{\rm proj}_\mu \|$.

\begin{lemma}
\label{lem:thdiff.tail.approx}
Under Assumptions \ref{asm:convex} and \ref{asm:approach.eps} and \ref{asm:Holder.eps} there exists  some $\mu\in[-R,R]$ such that the following holds with probability $1-\delta$:

\begin{equation*}
\|\wt \theta^{\rm proj}_\mu-\wt \theta_\mu \| \leq \frac {R\sqrt{ \frac{2\log (4/\delta) }{T}}+\upsilon+\zeta}{\|\E [ \phi(x) ] \|},
\end{equation*}

\end{lemma}

\begin{proof}
 $\wt \theta^{\rm proj}_\mu$ is the solution of following optimization problem:
\begin{equation*}
\wt \theta^{\rm proj}_\mu=\underset{\theta\in \R^{\wt p}}{\arg\min} \|\theta-\wh \theta_{\mu}\|^2\qquad\text{s.t.}\qquad\E [ h(x;\theta) ] =\mu_f,
\end{equation*}
where $\mu_f=\E [ f_c(x) ] $. Thus $\wt \theta^{\rm proj}_\mu$ can be obtain as the extremum of the following Lagrangian:

\begin{equation*}
\L(\theta,\lambda)=\|\theta-\wt \theta_\mu\|^2+\lambda(\E [ h(x;\theta) ] -\mu_f).
\end{equation*}

This problem can be solved in closed-form as follows:

\begin{align}
\label{eq:lagrange.approx}
0&=\frac{\partial \L(\theta,\lambda)}{\partial \theta}=\theta-\wt \theta_\mu+\lambda \E [ (\wt \phi(x) ]
\\ \notag
0&=\frac{\partial \L(\theta,\lambda)}{\partial \lambda}=\E [h(x;\theta) ]-\mu_f.
\end{align}

Solving the  above system of equations leads to $\E [ h(x;\wt \theta_\mu)] -\lambda \E [ \wt \phi(x) ]=\mu_f$. The solution for $\lambda$ can be obtained as
\begin{equation*} 
\lambda=\frac{\mu-\E [ h(x;\wt \theta_\mu) ]}{\|\E [ \wt \phi(x) ]\|^2}.
\end{equation*}

By plugging this in Eqn.~\ref{eq:lagrange.approx} we deduce:

\begin{equation*} 
\wt \theta^{\rm proj}_\mu=\wt \theta_{\mu}-\frac{(\mu_f-\E [ h(x;\wt \theta_\mu) ])\E [ \wt \phi(x) ]}{\|\E [\wt \phi(x) ]\|^2}, 
\end{equation*}

We then deduce:
\begin{equation*} 
\begin{aligned}
&\|\wt \theta^{\rm proj}_\mu-\wt \theta_{\mu}\|=\frac{|\mu_f-\E [ h(x;\wh \theta_\mu)| ]}{\|\E [ \wt \phi(x) ]\|}
\\
\leq& \frac{\E [ | f_c(x)- h(x;\theta^\zeta)| ]+ | \E [ h(x;\theta^\zeta) ] -\E [ h(x;\wh \theta_\mu) ] |}{\|\E [ \wt \phi(x) ]\|}.
\end{aligned}
\end{equation*}

This combined with Eqn.~\ref{eq:expect.diff.tail.approx} and a union bound proves the result.
\end{proof}

We proceed by proving bound on the absolute error   $|L(\wt \theta^{\rm proj}_\mu)-L(\theta_c)|=|L(\wt \theta^{\rm proj}_\mu)-\min_{\theta\in\wt \Theta}L(\theta)|$.

\begin{lemma}
\label{lem:Lwt.tail.approx}
Under Assumptions \ref{asm:convex}, \ref{asm:approach.eps} and \ref{asm:Holder.eps} there exists some $\mu\in[-R,R]$  such that for every $\zeta>0$ the following bound holds with probability $1-\delta$:
\begin{equation*}
\big| L(\wt \theta^{\rm proj}_\mu)-L(\theta_c)\big| = \O\left( \zeta+\upsilon+B R U\sqrt{ \frac{\log (1/\delta) }{T}}\right).
\end{equation*}
\end{lemma}

\begin{proof}
From Lem.~\ref{lem:thdiff.tail.approx} we deduce

\begin{equation}
\label{eq:diff.h.approx}
\begin{aligned} 
&|\E [ h(x;\wt \theta^{\rm proj}_\mu) -h(x;\wt \theta_\mu) ] | \\
&\leq \|\wt \theta^{\rm proj}_\mu-\wt \theta_{\mu}\|\|\E [ \wt \phi(x) ] \| \leq 2R\sqrt{\frac {\log(4/\delta)}{T}}+\zeta+\upsilon.
\end{aligned}
\end{equation}
where in the first inequality we rely  on the Cauchy-Schwarz inequality. We then deduce:
\begin{align*} 
&|~|L(\wt\theta^{\rm proj}_{\mu})-L(\theta_c)|-|L(\wt\theta_{\mu})-L(\theta_c)|~|\\
&\leq |L(\wt\theta^{\rm proj}_{\mu})-L(\wt\theta_{\mu})|\leq |\E [ h(x;\wt \theta^{\rm proj}_\mu) -h(x;\wh \theta_\mu) ] |,
\end{align*}
in which we rely on the triangle inequality $|~|a|-|b|~|\leq |a-b|$. We then deduce

\begin{equation*}
\begin{aligned}
L(\wt\theta^{\rm proj}_{\mu})-L(\theta_c)&\leq |L(\wh\theta_{\mu})-L(\theta_c)|
\\
&+ |\E [ h(x;\wt \theta^{\rm proj}_\mu) -h(x;\wt \theta_\mu) ] |.
\end{aligned}
\end{equation*}

Combining this result with the result of Lem.~\ref{lem:Lalp.tail.approx} and Eqn.~\ref{eq:diff.h.approx} proves the main result.

\end{proof}

In the  following lemma, we make use of Lem.~\ref{lem:thdiff.tail.approx} and Lem.~\ref{lem:Lwt.tail.approx} to prove that   the minimizer  $\wh x_\mu={\arg\min}_{x\in\X} h(x;\wh \theta_{\mu})$ is near a global minimizer $x^*\in \X^*_f$ w.r.t. to the metric $d$.

\begin{lemma}
\label{lem:x.bound.approx}
Under Assumptions  \ref{asm:convex}, \ref{asm:approach.eps} and \ref{asm:Holder.eps} there exists some  $\mu\in[-R,R]$   such that w.p. $1-\delta$:

\begin{align*}
d( \wh x_\mu, \X^*_f) =\O \left[\left(\sqrt{\frac{\log (1/\delta) }{T}}+\zeta+\upsilon\right)^{\beta_1\beta_2}\right].
\end{align*}

\end{lemma}

\begin{proof}
The result of Lem.~\ref{lem:Lwt.tail.approx} combined with Assumption~\ref{asm:Holder.eps}.b implies that w.p. $1-\delta$:

\begin{equation*}
d_2(\theta^{\rm proj}_\mu,\Theta_c ) \leq \left(\frac{\eps_1(\theta)}{\gamma_2}\right)^{\beta_2},
\end{equation*}
where $\eps_1(\theta)=\O(B R U\sqrt{ \frac{\log (1/\delta) }{T}}+\upsilon+\zeta)$. This combined with the result of Lem.~\ref{lem:thdiff.tail.approx} implies that w.p. $1-\delta$:
\begin{equation*}
d_2(\wt \theta_\mu,\theta_c) \leq d_2(\wt \theta^{\rm proj}_\mu,\theta_c) + d_2(\wt \theta^{\rm proj}_\mu, \wt \theta_\mu) \leq 2\left( \frac{\eps_c(\delta)}{\gamma_2}\right)^{\beta_2},
\end{equation*}
where $\eps_c(\delta)$ is defined as:

\begin{equation*}
\eps_c(\delta):=\O\left( \frac{RB U\sqrt{\frac{\log  (1/\delta)}{ T}}+\zeta+\upsilon}{\min(1,\|\E [ \wt \phi(x)))\| ]} \right).
\end{equation*}

We now use this result to prove high probability bound on $f_c(\wh x_\mu)-f^*$ :
\begin{equation*}
\begin{aligned}
&f_c(\wh x_\mu)-f^*=h(\theta_c,\wh  x_{\mu})- h(\theta_c, x^*) \\
&=h(\theta_c,\wh  x_\mu)- h(\wh \theta_\mu, \wh x_\mu)+  \min_{x\in\X} h(\wh \theta_\mu, x)- h(\theta_c,  x^*)\\
&\leq h(\theta_c,\wh  x_\mu)- h(\wh \theta_\mu, \wh x_\mu)+  h(\wh \theta_\mu,x^*)- h(\theta_c,  x^*)\\
&\leq 2U d_2(\wh \theta_{\mu},\Theta_c)\leq 2\gamma_2 U \left(\frac{\eps_{c}(\delta) }{\gamma_2}\right)^{\beta_2},
\end{aligned}
\end{equation*}
where the last inequality follows by  the fact that $h$ is U-Lipschitz w.r.t. $\theta$. This combined with Assumption~\ref{asm:Holder.eps}.a completes the proof.

\end{proof}

It then follows by combining the result of  Lem.~\ref{lem:x.bound.approx} and Assumption~\ref{asm:smoothness} that there exist a $ \mu \in [-R, R]$ such that for every $\xi>0$:

\begin{equation*}
f(\wh x_\mu)-f^*= \O \left[ \left(\sqrt{\frac{\log (1/\delta) }{T}}+\upsilon+\xi\right)^{\beta_1\beta_2}\right]
\end{equation*}

This combined with the fact that  $f(\wh x_{\wh\mu})\leq f(\wh x_{\mu})$ for every $ \mu \in [-R, R]$ completes the proof of the main result  (Thm.~\ref{thm:f.bound.approx}) .



\end{document}